\documentclass[12pt,journal,letterpaper,onecolumn]{IEEEtran}

\usepackage{amsthm}    

\usepackage{amsmath,amssymb,amscd,latexsym,dsfont,wasysym,bbm}
\usepackage{float}
\usepackage{color}
\usepackage{graphicx,subfigure}
\usepackage{multirow,multicol} 
\usepackage{verbatim}
\usepackage{psfrag}
\usepackage{comment}
\usepackage{makeidx}
\usepackage[]{algorithm}
\usepackage{algorithmicx}
\usepackage{algpseudocode}
\usepackage{cases}
 \usepackage{setspace}
\usepackage{pgfplots} 
\newcommand{\tr}[1]{\textrm{#1}}
\newcommand{\mr}[1]{\mathrm{#1}}
\newcommand{\tnr}[1]{{\textnormal{#1}}}

\newcommand{\mc}[1]{\mathcal{#1}}

\newcommand{\ms}[1]{\mathds{#1}}

\newcommand{\ov}[1]{\overline{#1}}

\newcommand{\un}[1]{\underline{#1}}

\newcommand{\bI}{\boldsymbol{I}}

\newcommand{\bV}{\boldsymbol{V}}

\newcommand{\bu}{\boldsymbol{u}}

\newcommand{\bv}{\boldsymbol{v}}

\newcommand{\bx}{\boldsymbol{x}}

\newcommand{\by}{\boldsymbol{y}}

\newcommand{\bzero}{\boldsymbol{0}}
\newcommand{\bone}{\boldsymbol{1}}

\newcommand{\btheta}{\boldsymbol{\theta}}

\newcommand{\bmu}{\boldsymbol{\mu}}

\newcommand{\figref}[1]{Fig.~\ref{#1}}

\newcommand{\secref}[1]{Sec.~\ref{#1}}
\newcommand{\appref}[1]{Appendix~\ref{#1}}

\newcommand{\propref}[1]{Proposition~\ref{#1}}
\newcommand{\tabref}[1]{Table~\ref{#1}}

\newcommand{\ie}{i.e.,~} 		
\newcommand{\eg}{e.g.,~}	

\newcommand{\argmax}{\mathop{\mr{argmax}}}

\newcommand{\set}[1]{\{#1\}}
\newcommand{\SET}[1]{\left\{#1\right\}}

\newcommand{\cd}{\cdot}
\newcommand{\ld}{\ldots}

\newcommand{\e}{\mr{e}}

\newcommand{\PR}[1]{\Pr\SET{#1}}       	
\newcommand{\pdf}{f}            			

\newcommand{\IND}[1]{\ms{I}\big[{#1}\big]}   	
\newcommand{\Ex}{\ms{E}}     			
\newcommand{\T}{^{\mr{T}}}            		
\newcommand{\dd}{\,\mr{d}}             		

\newcommand{\mcI}{\mc{I}}
\newcommand{\mcJ}{\mc{J}}

\newcommand{\mcN}{\mc{N}}

\newcommand{\mcP}{\mc{P}}

\newcommand{\mcY}{\mc{Y}}

\newcommand{\sizf}{0.8}
\newcommand{\sizfs}{0.65}   

\usepackage[acronym,nonumberlist]{glossaries} 
\usepackage{hyperref}
\usepackage{empheq} 

\newacronym[\glsshortpluralkey=PDFs,\glslongpluralkey=probability density functions]{pdf}{PDF}{probability density function}
\newacronym[\glsshortpluralkey=CDFs,\glslongpluralkey=cumulative density functions]{cdf}{CDF}{cumulative density function}
\newacronym[\glsshortpluralkey=CCDFs,\glslongpluralkey=complementary cumulative density functions]{ccdf}{CDF}{complementary cumulative density function}
\newacronym[\glsshortpluralkey=PMFs,\glslongpluralkey=probability mass functions]{pmf}{PMF}{probability mass function}
\newacronym[]{lhs}{l.h.s.}{left-hand side}
\newacronym[]{rhs}{r.h.s.}{right-hand side} 

\newacronym[]{bicm}{BICM}{bit-interleaved coded modulation}
\newacronym[]{bicmid}{BICM-ID}{BICM with iterative demapping}
\newacronym[]{cm}{CM}{coded modulation}
\newacronym[]{tcm}{TCM}{trellis-coded modulation}
\newacronym[]{mlc}{MLC}{multi-level coding}
\newacronym[]{pam}{PAM}{pulse amplitude modulation}
\newacronym[]{bpsk}{BPSK}{binary phase shift keying}
\newacronym[]{qam}{QAM}{quadrature amplitude modulation}
\newacronym[]{16qam}{16-QAM}{16-points quadrature amplitude modulation}
\newacronym[]{psk}{PSK}{phase shift keying}
\newacronym[\glsshortpluralkey=LLRs,\glslongpluralkey=logarithmic likelihood ratios]{llr}{LLR}{logarithmic likelihood ratio}
\newacronym[]{oc}{OC}{operating characteristic}
\newacronym[]{map}{MAP}{maximum a posteriori}
\newacronym[]{ml}{ML}{maximum likelihood}
\newacronym[]{dmp}{DMP}{discretized message passing}
\newacronym[]{mp}{MP}{message passing}
\newacronym[]{ep}{EP}{expectation propagation}
\newacronym[\glsshortpluralkey=MIs,\glslongpluralkey=mutual informations]{mi}{MI}{mutual information}
\newacronym[\glsshortpluralkey=GMIs,\glslongpluralkey=generalized mutual informations]{gmi}{GMI}{generalized mutual information}
\newacronym[]{eesm}{EESM}{exponential effective-SNR-mapping}
\newacronym[]{bicm-gmi}{BICM-GMI}{BICM generalized mutual information}
\newacronym[]{awgn}{AWGN}{additive white Gaussian noise}
\newacronym[]{bsc}{BSC}{binary symetric channel}
\newacronym[]{amc}{AMC}{adaptive modulation and coding}
\newacronym[]{csi}{CSI}{channel state information}
\newacronym[]{cqi}{CQI}{channel quality indicator}
\newacronym[]{kl}{KL}{Kullback-Leibler}
\newacronym[]{cmm}{CMM}{circular moment matching}
\newacronym[]{ga}{GA}{Gaussian approximation}

\newacronym[]{sp}{SP}{set-partitioning}
\newacronym[]{gsm}{GSM}{global system for mobile communications}
\newacronym[]{edge}{EDGE}{enhanced data rates for GSM evolution}
\newacronym[]{3gpp}{3GPP}{3rd generation partnership project}
\newacronym[]{umts}{UMTS}{Universal Mobile Telecommunication System}
\newacronym[]{lte}{LTE}{Long Term Evolution}
\newacronym[]{dvb}{DVB}{digital video broadcasting}
\newacronym[]{fdd}{FDD}{Frequency Division Duplexing}

\newacronym[\glsshortpluralkey=CCs,\glslongpluralkey=convolutional codes]{cc}{CC}{convolutional code}
\newacronym[\glsshortpluralkey=PCCCs,\glslongpluralkey=parallel concatenated convolutional codes]{pccc}{PCCC}{parallel concatenated convolutional code}
\newacronym[\glsshortpluralkey=TCs,\glslongpluralkey=turbo codes]{tc}{TC}{turbo code}
\newacronym{ldpc}{LDPC}{low-density parity-check}
\newacronym[]{ofdm}{OFDM}{orthogonal frequency-division multiplexing}
\newacronym[]{bep}{BEP}{bit-error probability}
\newacronym[]{wep}{WEP}{word-error probability}
\newacronym[]{sep}{SEP}{symbol-error probability}
\newacronym[]{pep}{PEP}{pairwise-error probability}
\newacronym[]{ttcm}{TTCM}{turbo-trellis coded modulation}
\newacronym[]{uep}{UEP}{unequal error protection}
\newacronym[\glsshortpluralkey=CENCs,\glslongpluralkey=convolutional encoders]{cenc}{CENC}{convolutional encoder}
\newacronym[]{mimo}{MIMO}{multiple-input multiple-output}
\newacronym[\glsshortpluralkey=SNRs,\glslongpluralkey=signal-to-noise ratios]{snr}{SNR}{signal-to-noise ratio}
\newacronym[\glsshortpluralkey=SINRs,\glslongpluralkey=the signal-to-interference-plus-noise ratios]{sinr}{SINR}{the signal-to-interference-plus-noise ratio}
\newacronym[]{msb}{MSB}{most-significative bit}
\newacronym[]{bcjr}{BCJR}{Bahl--Cocke--Jelinek--Raviv}
\newacronym[]{cbc}{CBC}{Colavolpe--Barbieri--Caire}
\newacronym[]{skr}{SKR}{Shayovitz--Kreimer--Raphaeli}
\newacronym[\glsshortpluralkey=SEDs,\glslongpluralkey=squared Euclidean distances]{sed}{SED}{squared Euclidean distance}
\newacronym[\glsshortpluralkey=EDs,\glslongpluralkey=Euclidean distances]{ed}{ED}{Euclidean distance}
\newacronym[\glsshortpluralkey=MEDs,\glslongpluralkey=minimum Euclidean distances]{med}{MED}{minimum Euclidean distance}
\newacronym[]{core}{CoRe}{constellation rearrangement}
\newacronym[]{msd}{MSD}{multistage decoding}
\newacronym[]{pdl}{PDL}{parallel decoding of the individual levels}
\newacronym[\glsshortpluralkey=GCs,\glslongpluralkey=Gray codes]{gc}{GC}{Gray code}
\newacronym[]{brgc}{BRGC}{binary-reflected Gray code}
\newacronym[]{nbc}{NBC}{natural binary code}
\newacronym[]{fbc}{FBC}{folded-binary code}
\newacronym[]{bsgc}{BSGC}{binary semi-Gray code}
\newacronym[]{msp}{MSP}{modified set-partitioning}
\newacronym[]{ssp}{SSP}{semi set-partitioning}
\newacronym[]{fhd}{FHD}{free Hamming distance}
\newacronym[]{mfhd}{MFHD}{maximum free Hamming distance}
\newacronym[]{ods}{ODS}{optimal distance spectrum}
\newacronym[]{iud}{i.u.d.}{independent and uniformly distributed}
\newacronym[]{ud}{u.d.}{uniformly distributed}
\newacronym[]{iid}{i.i.d.}{independent, identically distributed}
\newacronym[]{ami}{AMI}{accumulated mutual information}
\newacronym[]{bico}{BICO}{binary-input continuous-output}
\newacronym[]{gh}{GH}{Gauss--Hermite}
\newacronym[]{gum}{GUM}{Gaussian--uniform mixture}

\newacronym[\glsshortpluralkey=BSs,\glslongpluralkey=base-stations]{bs}{BS}{base-station}
\newacronym[\glsshortpluralkey=MSs,\glslongpluralkey=mobile-stations]{ms}{MS}{mobile-stations}

\newacronym[]{phy}{PHY}{physical layer} 
\newacronym[]{rlc}{RLC}{Radio-Link control} 
\newacronym[]{ran}{RAN}{Radio Access Network} 
\newacronym[]{llc}{LLC}{logical link control} 
\newacronym[]{tcp}{TCP}{transmission control protocol} 
\newacronym[]{mac}{MAC}{media access control} 
\newacronym[]{fft}{FFT}{fast Fourier transform} 
\newacronym[]{ft}{FT}{Fourrier transform}
\newacronym[]{cf}{CF}{characteristic function} 
\newacronym[]{mgf}{MGF}{moment generating function} 
\newacronym[]{ee}{EE}{energy efficiency} 
\newacronym[]{eb}{EB}{energy per bit}
\newacronym[]{kkt}{KKT}{Karush--Kuhn--Tucker} 
\newacronym[]{mcs}{MCS}{modulation/coding scheme} 
\newacronym[]{fec}{FEC}{forward error correction}
\newacronym[]{arq}{ARQ}{automatic repeat request}
\newacronym[]{harq}{HARQ}{hybrid ARQ}
\newacronym[]{tarq}{TARQ}{truncated HARQ}
\newacronym[]{ir}{IR}{incremental redundancy}
\newacronym[]{rpr}{RR}{repetition redundancy}
\newacronym[]{rrharq}{RR-HARQ}{repetition redundancy HARQ}
\newacronym[]{irharq}{IR-HARQ}{incremental redundancy HARQ}
\newacronym[]{ack}{ACK}{positive acknowledgment}
\newacronym[]{nack}{NACK}{negative acknowledgment}
\newacronym[]{hol}{HoL}{head of the line}
\newacronym[]{crc}{CRC}{cyclic redundancy check}
\newacronym[]{dp}{DP}{dynamic programming}
\newacronym[]{gp}{GP}{geometric programming}
\newacronym[]{per}{PER}{packet error rate}
\newacronym[]{ber}{BER}{bit error rate}
\newacronym[]{op}{OP}{outage probability}
\newacronym[]{spa}{SPA}{saddle-point approximation}
\newacronym[]{mrc}{MRC}{maximum ratio combining}
\newacronym[]{mdp}{MDP}{Markov decision process}
\newacronym[]{lp}{LP}{linear programming}
\newacronym[]{pomdp}{POMDP}{partially observable Markov decision process}
\newacronym[]{psimdp}{PSI-MDP}{partial state information Markov decision process}
\newacronym[]{scpp}{SCPP}{stochastic shortest path problem}

\newacronym[]{forw}{frwd}{forward}
\newacronym[]{feed}{fdbk}{feedback}

\newacronym[]{mm}{MM-HARQ}{multi-message HARQ}
\newacronym[]{xp}{XP-HARQ}{cross-packet HARQ}
\newacronym[]{ts}{TS}{time-sharing}
\newacronym[]{sc}{SC}{superposition coding}
\newacronym[]{sbrq}{SBRQ}{systematic backward retransmission}
\newacronym[]{brq}{BRQ}{backward retransmission}
\newacronym[]{lharq}{L-HARQ}{layer-coded HARQ}
\newacronym[]{anlharq}{AoN-HARQ}{all-or-none L-HARQ}
\newacronym[]{vlharq}{VL-HARQ}{variable-length HARQ}

\newacronym[]{pp}{PPP}{point process}
\newacronym[]{ppp}{PPP}{Poisson point process}
\newacronym[]{pgfl}{PGFL}{Poisson point process}

\newacronym[]{fide}{FIDE}{F\'ed\'eration Internationale des \'Echecs}
\newacronym[]{fifa}{FIFA}{F\'ed\'eration Internationale de Football Association}
\newacronym[]{epl}{EPL}{English Premier Ligue}
\newacronym[]{nhl}{NHL}{National Hockey Ligue}
\newacronym[]{nfl}{NFL}{National Football Ligue}
\newacronym[]{sg}{SG}{stochastic gradient}
\newacronym[]{lms}{LMS}{least mean squares}
\newacronym[]{rls}{RLS}{recursive least squares}
\newacronym[]{vss}{VSS}{variable step-size}
\newacronym[]{hfa}{HFA}{home-field advantage}
\newacronym[]{mov}{MOV}{margin of victory}
\newacronym[]{ac}{AC}{Adjacent Categories}
\newacronym[]{cl}{CL}{Cumulative Link}
\newacronym[]{skf}{SKF}{Simplified Kalman Filter}
\newacronym[]{vskf}{vSKF}{\emph{vector-covariance} Simplified Kalman Filter}
\newacronym[]{sskf}{sSKF}{\emph{scalar-covariance} Simplified Kalman Filter}
\newacronym[]{fskf}{fSKF}{\emph{fixed-variance} Simplified Kalman Filter}
\newacronym[]{kf}{KF}{Kalman Filter}

\newacronym[]{tpb}{TPB}{tensor-product-basis}

\newtheorem{proposition}{Proposition}

\begin{document}


\title{Simplified Kalman filter for online rating: \\one-fits-all approach}

\author{Leszek Szczecinski\thanks{L. Szczecinski is with 
Institut National de la Recherche Scientifique, Montreal, Canada. e-mail: leszek@emt.inrs.ca.} 
~ and 
Rapha\"elle Tihon\thanks{%
R. Tihon is with
University of Montreal, Canada. e-mail: raphael.tihon@umontreal.ca.}
}%


\maketitle

\begin{abstract}
In this work, we deal with the problem of rating in sports, where the skills of the players/teams are inferred from the observed outcomes of the games. Our focus is on the online rating algorithms which estimate the skills after each new game by exploiting the probabilistic models of the relationship between the skills and the game outcome. We propose a Bayesian approach which may be seen as an approximate Kalman filter and which is generic in the sense that it can be used with any skills-outcome model and can be applied in the individual- as well as in the group-sports. We show how the well-know algorithms (such as the Elo, the Glicko, and the TrueSkill algorithms) may be seen as instances of the one-fits-all approach we propose. In order to clarify the conditions under which the gains of the Bayesian approach over the simpler solutions can actually materialize, we critically compare the known and the new algorithms by means of numerical examples using the synthetic as well as the empirical data.
\end{abstract}


\section{Introduction}\label{Sec:Introduction}

The rating of the players\footnote{We will talk about \emph{players} but the team sports are of course treated in the same way. In fact, the examples we provide come from team-sports but the advantage of talking about players is that it allows us to discuss the issue of gathering players in groups which face each other, as done \eg in eSports \cite{Herbrich06}.} 
is one of the fundamental problem in sport analytics and consists in assigning each player a real value called a \emph{skill}. In this work we are interested in the rating algorithms that can be systematically derived from the probabilistic models which describe i)~how the the skills affect the outcomes of the games, as well as ii)~how the skills  evolve in time, \ie characterize the skills \emph{dynamics}. Using the probabilistic models, the forecasting of the game outcomes is naturally derived from the rating. 

Once the models are chosen, the rating boils down to inferring the unknown skills from the observed games outcomes. This is essentially a parameter estimation problem which  has been largely addressed in the literature. In particular, using a static model for the skills, that is, assuming that the skills do not vary in time, the problem consists in solving a non-linear regression problem and the main issue then is to define a suitable skills-outcome model. 

The most popular skills-outcome models are obtained from the pairwise comparison framework which is well known in the psychometrics literature \cite{David63_Book}, \cite{Cattelan12}.  For binary games (win/loss), the Thurston model \cite{Thurston27} and Bradley-Terry model \cite{Bradley52} are the most popular. Their extensions to the ternary games (win/loss/draw) were proposed in \cite{Rao67} and \cite{Davidson70}; these are particular cases of ordinal variable models  \cite{Agresti92} that may may be applied in multi-level games as done \eg in \cite{Fahrmeir94} and \cite{Goddard05}.

Alternative approach focuses on modelling directly the game points (\eg goals) using predefined distributions; the Poisson distribution is the most popular in this case \cite{Maher82}; similarly, the points difference can be modelled using, \eg the Skellam, \cite{Karlis08} or the Weibull \cite{Boshnakov17} distributions.

The very meaning of the skills may be also redefined and instead of a scalar, the player may be assigned two values corresponding to offensive and defensive skills \cite{Maher82}, \cite{Manderson18}, \cite{Wheatcroft20a}; further, considering the \gls{hfa}, three or four distinct parameters per player may be defined \cite{Maher82}, \cite{Kuk95}, although recent results indicate that this may lead to over-fitting \cite{Ley19}, \cite{Lasek20}.

The various skills-outcome model we mention above, invariably assume that the outcome of the game depends (via a non-linear function) on a linear combination of the skills of the participating players, this assumption is  also used in the case when the multiple players are gathered in two groups facing each other in a game \cite{Herbrich06}. This general approach  will be also used as a basis for our work.

Each of these skills-outcome models can be combined with the models describing how the skills evolve in time. With that regard, the most popular is modelling of the skills via first-order Markov Gaussian processes, \eg \cite{Fahrmeir92}, \cite{Glickman93_thesis}, \cite{Fahrmeir94}, \cite{Glickman99}, \cite{Knorr00}, \cite{Held05}, \cite{Herbrich06}, \cite{Koopman12}, \cite{Manderson18}, \cite{Koopman19}. This formulation is then exploited to derive the  on-line rating algorithms in two recursive steps: first, at a time $t$, the posterior distribution of the skills is found using all observations up to time $t$; more precisely, to simplify the problem, the Gaussian approximation of the latter is obtained. Next, the posterior distribution from the time $t$, is used as a prior in the time $t+1$. This approach should be seen as a generalization of  the Kalman filtering to the non-Gaussian models characteristic of the rating problems \cite{Fahrmeir92}.

Most of the works we cited above and which consider the skills' dynamics, focused on the estimation of the skills with a moderate number of players, the case which is typical in sport leagues (\eg in football, hockey, American football, etc). In such a case, the approximate (Gaussian) posterior  distribution of the skills can be fully defined by the mean vector and by the covariance matrix. 

On the other hand, for large number of players (\eg thousands of chess players or millions of eSports players), it is considered unfeasible and further approximations are introduced by considering only a diagonal covariance matrix, equivalent to assuming that the skills are, a posteriori, Gaussian and independent. This approach was proposed to rate the chess players (the Glicko algorithm \cite{Glickman99}) as well as for the rating in eSports (the TrueSkill algorithm, \cite{Herbrich06}).


However, the Glicko and the TrueSkill algorithms are derived i) from different skills-outcome models,\footnote{Glicko uses the Bradley-Terry model \cite{Bradley52}, while TrueSkill uses the Thurston model \cite{Thurston27}.} and ii) using different approximation principles. Thus, not only it may be difficult to see them as instances of a more general approach but, more importantly, they cannot be straightforwardly reused to obtain new online ratings if we would like to change the skills-outcome model.

This latter fact stays very much in contrast with the approach of \cite{Fahrmeir92} which is general in its formulation so, under mild conditions, it can be applied to any skills-outcomes model. However, since the focus of \cite{Fahrmeir92} and other works which followed its path, was not on the large problems, the derivations did not leverage the simplifying assumptions on which rely the TrueSkill and Glicko algorithms.

In our work we thus want to take advantage of both worlds: we will exploit the independence assumption on which  \cite{Glickman99} and \cite{Herbrich06} are built, and the estimation principle underlying the Kalman filter which was used in \cite{Fahrmeir92}. Furthermore, we will also consider new simplifying assumptions about the posterior distributions of the skills; this will lead to different simplified versions of the Kalman filter.

The goal of this work is thus threefold:
\begin{itemize}
    \item We will show how the online rating algorithms can be derived for \emph{any} skills-outcome model and may be applied equally well to estimate the skills of the players in individual sports (as in the Glicko algorithms) or the skills of the players withing a group (as in the True-Skill algorithm). We will  also consider different level of simplification when dealing with the skills' dynamics. 
    \item Using this generic algorithmic framework, we will be able not only to derive new algorithms, but also to compare and understand the similarities/differences between the known online rating such as the Elo \cite{Elo08_Book}, the Glicko \cite{Glickman99}, or the TrueSkill \cite{Herbrich06} algorithms.
    \item By mean of numerical examples, we will provide an insight into the advantages of the simplified versions of the rating algorithms, and indicate under which conditions the simple rating algorithm may perform equally well as the more complex ones.
\end{itemize}

The paper is organized as follows:  the model underlying the rating is shown in \secref{Sec:Model}. The on-line rating algorithms are derived in a general context in \secref{Sec:Tracking} using different approximations of the posterior distribution. In \secref{Sec:New.Ratings}, the  popular scalar skills-outcome model are used to derive new rating algorithms which are then compared to the algorithms from the literature.  Numerical examples are shown in \secref{Sec:Num.results} and conclusions are draws in \secref{Sec:Conclusions}.

\section{Model}\label{Sec:Model}
We consider the case when the players indexed with $m\in\set{1,\ld, M}$ participate in the games indexed with $t\in\set{1,\ld, T$}, where the number of games, $T$, is finite (as in sport seasons) or infinite (as in non-stop competitions, such as eSports). 

We consider one-on-one games: in the simplest case it means that the ``home'' player $i_t$ plays against the ``away'' player $j_t$, where $i_t, j_t\in\set{1,\ld,M}$. The outcomes of the game $t$, denoted by $y_t$, belong to an ordinal set $\mcY$. For example, if $y_t$ is the difference between the game-points (such as goals), we have  $\mcY=\set{\ld-3,-2,-1,0,1,2,\ld}$ so the ordinal variables are naturally encoded into integers. On the other hand, in ternary-outcome games, we may assign $y_t=0$ if the player $i_t$ loses the game, $y_t=2$ if  she wins the game, and $y_t=1$ if the game ends in a draw, \ie $\mcY=\set{0,1,2}$. The very notion of the home/away players is useful when dealing with the \acrfull{hfa} typically encountered in sports but it also helps us to ground the meaning of the game outcome: even in the absence of the \gls{hfa}, for the outcome $y_t=0$ to be meaningful, we must decide which player lost (the home player in our notation for ternary-outcome games).

In a more general setup, the game may implicate two \emph{groups} of players whose indices are defined by the sets $\mcI_t=\set{i_{t,1}, i_{t,2},\ld, i_{t,F}}$ (these are indices of the ``home'' players) and $\mcJ_t=\set{j_{t,1}, j_{t,2},\ld, j_{t,F}}$ (indices of the ``away'' players). While the number of players in each group, $F$, is assumed to be constant, this is done merely to simplify the notation and other cases are possible. For example, the groups in eSports are formed on the fly and they do not  always have the same number of players, nor even the same number of players in both groups playing against each other \cite{Herbrich06}.  The process of defining which players take part in the game $t$, \ie how the indices $i_t$, $j_t$ (or the sets $\mcI_t$ and $\mcJ_t$) are defined, is called a \emph{scheduling}.

The above notation applies directly if we replace the notion of ``player'' with ``team'';  but then, of course, the general case of groups  defined by $\mcI_t$ and $\mcJ_t$ is not necessary. For the rest of the work we only refer to players which is a more general case to deal with.

In its simplest form, the rating consists in assigning the player $m$ the value $\theta_{t,m}$, called a \emph{skill}; this is done after observing the outcomes of the games up to time $t$, which we denote as $\un{y}_t=\set{\un{y}_{t-1},y_t}$. We index the skills with $t$ because we assume that they may vary in time. 

The probabilistic perspective we will adopt relies on the model relating the skills $\btheta_t=[\theta_{t,1},\ld, \theta_{t,M}]\T$ to $\un{y}_t$; it comprises i) the skills-outcome model, which makes explicit the relationship  between the skills $\btheta_t$ and the outcome $y_t$ at time $t$,  as well as ii) the model describing the evolution of the skills in time, \ie the skills' dynamics. 

The problem of finding the skills-outcome model has been treated extensively in the literature often exploitng the link with the problem of pairwise comparison well studied in psychometry \cite{Thurston27}, \cite{Bradley52}. This is also where most of the efforts are concentrated in the rating literature and many models have been alrady proposed and studied. On the other hand, the modelling of the dynamics of the skills is less diversified and mainly focuses on applying the particular skills-outcome model in the dynamic context.

~

\textbf{Skills-outcome model}

The skills-outcome model defines the probability of the outcome conditioned on the skills, $\PR{y_t|\btheta_t}$, and  most often is defined by combing the non-linear scalar function and the linear function of the skills, \ie
\begin{align}\label{pdf.y.theta}
\PR{y_t|\btheta_t}&=L(z_t/s; y_t),\\
\label{z.t}
z_t&=\bx_{t,\tr{h}}\T\btheta_t -\bx_{t,\tr{a}}\T\btheta_t = \bx\T_t\btheta_t,
\end{align}
where the role of $s>0$ is to scale the values the skills; $\bx_{t,\tr{h}}=[x_{t,\tr{h},1},\ld, x_{t,\tr{h},M}]\T$ is the home scheduling vector, defined as follows: $x_{t,\tr{h}, m}=1$ if the player $m$ is a home player, \ie $m\in\mcI_t$, and $x_{t,\tr{h},m}=0$, otherwise. The away scheduling vector is defined correspondingly for the away players.

For example, if $M=10$, and 
$\bx_{t,\tr{h}} = [ 0, 0, 0, 1, 0, 0, 1, 0, 0, 0]\T$, and $\bx_{t,\tr{a}} = [ 0, 0, 1, 0, 0, 0, 0, 0, 0, 1]\T$, it means that, 
at time $t$, the game involves the home players $\mcI_t=\set{4,7}$ and away players $\mcJ_t=\set{3,10}$. Then, the combined scheduling vector $\bx_t=\bx_{t,\tr{h}}-\bx_{t,\tr{a}}$ is given by $\bx_{t} = [ 0, 0, -1, 1, 0, 0, 1, 0, 0, -1]\T$.

Of course, we do not suggest that the vector product(s) in \eqref{z.t} should be actually implemented; it is just a convenient notation expressing the fact that $z_t$ is the difference between the sum of the skills of the home players and the sum of the skills of the away players. 

As for the function $L(z;y)$ it should be defined taking into account the structure of the space of outcomes $\mcY$. For example, in the binary games, $\mcY=\set{0,1}$, we often use
\begin{align}\label{L.example}
L(z;y)=
\begin{cases}
F(z) & \tr{if}\quad y=1 \quad \tr{(home win)}\\
F(-z) & \tr{if}\quad y=0 \quad \tr{(away win)}
\end{cases},
\end{align}
where $0\le F(z)\leq 1$ is a non-decreasing function. This corresponds to the assumption that increasing the difference between the skills, $z_t$, corresponds to the increased probability of the home win and, of course, decreased probability of the away win. More on that in \secref{Sec:New.Ratings}.

~

\textbf{Skills' dynamics}

The temporal evolution of the skills is often modelled as a damped random walk
\begin{align}\label{btheta.t.t1}
    \btheta_{t} =\beta_{t} \btheta_{t-1} + \bu_{t}\epsilon_{t},
\end{align}
where $\bu_t$ is the vector comprised of independent, zero-mean, unitary-variance, random Gaussian variables,  so
$\epsilon_{t}$  has the meaning of the variance of the random increment in skills from time $t-1$ to $t$: it is assumed to be the same  for all the player.  For example, \cite{Glickman99} uses 
\begin{align}\label{epsilon.t}
\epsilon_{t}= \big(\tau(t)-\tau(t-1)\big)\epsilon,
\end{align}
where $\tau(t)$ is the time (\eg measured in days) at which the game indexed with $t$ is played, and $\epsilon$ is the per-time unit increase of the variance.

The autoregression parameter $0<\beta_t\le 1$ models the decrease of the skills in time (in absence of game outcomes). While $\beta_{t}=1$ is used in most of the previous works we cite, $\beta_{t}<1$ was also used, \eg in \cite{Koopman12}, \cite{Manderson18}, \cite{Koopman19}, and to take into account the time we may define it as
\begin{align}\label{beta.t}
\beta_{t}= \beta^{(\tau(t)-\tau(t-1))},
\end{align}
where, again $\beta$ is the per-time decrease of the skills.

The relationship \eqref{btheta.t.t1} may be presented as
\begin{align}\label{dumped.Markov}
\pdf( \btheta_{t} | \btheta_{t-1} ) = \mcN(\btheta_t; \beta_t \btheta_{t-1},  \bI \epsilon_{t}),
\end{align}
where $\bI$ is the identity matrix and
\begin{align}\label{pdf.Normal}
\mcN( \btheta;  \bmu , \bV )=\frac{1}{\sqrt{\tr{det}(2\pi\bV)}}\exp\left( -\frac{1}{2} (\btheta-\bmu)\T\bV^{-1}(\btheta-\bmu)   \right)
\end{align}
is the multivariate Gaussian \gls{pdf} with the mean vector $\bmu$ and the covariance matrix $\bV$. 

As an alternative to \eqref{dumped.Markov} we may also assume that the skills in $\btheta_{t}$ (conditioned on $\btheta_{t-1}$) are correlated, \ie the covariance matrix has non-zero off-diagonal elements. This is done, \eg in \cite{Knorr00}, \cite{Manderson18} in order to ensure that $\bone\T\btheta_{t}=\tr{Const.}$, \ie that the skills at a given time $t$, sum up to the same constant. 

However, a direct consequence of the correlation between the skills is that, the outcome $y_t$ will affect not only the skills of the players involved in the game at time $t$ but also of all other players. This may result in rating which is slightly counter-intuitive but also, in the case of eSports, when the pool of the players is not predefined, this model may be difficult to justify. 

We will thus use the model \eqref{dumped.Markov} but we emphasize that our goal is not to justify particular assumptions underlying the skills-outcome model or the models for the skills' dynamics. We rather want to present a common framework which will i) show the relationship between the algorithms already known from the literature, and ii) allow us to create new online rating algorithms in a simple/transparent manner.

\section{Estimation of the skills}\label{Sec:Tracking}

If we suppose momentarily that the skills do not vary in time, \ie $\btheta_{t}=\btheta$ (or $\epsilon=0$), the problem of finding the skills may be formulated under the \gls{ml} principle
\begin{align}\label{ML.estimation}
    \hat{\btheta}&=\argmax_{\btheta} \PR{ \un{y}_T | \btheta }
    =\argmax_{\btheta} \prod_{t=1}^{T}  L(z_t/s; y_t)\\
    &=\argmax_{\btheta} \sum_{t=1}^{T}  \ell(z_t/s;y_t),
\end{align}
where $\ell(z_t/s; y_t)=\log L(z_t/s; y_t)$ is the log-likelihood, and we assumed that the observations $y_t$, when conditioned on the skills' difference, $z_t$, are independent.

The solution \eqref{ML.estimation} can be found uniquely if the log-likelihood $\ell(\bx_{t}\T\btheta/s; y_t)$ is a concave function of $\btheta$ which holds if $\ell(z; y_t)$ is concave in $z$. This is the ``mild'' condition we referred to in \secref{Sec:Introduction} and, in the rest of the work, we assume that this condition is satisfied.\footnote{For that, it is necessary that $\forall z, L''(z;y)L(z;y)\le [L'(z;y)]^2$.}

Unlike the \gls{ml} approach which finds a \emph{point} estimate of the skills $\hat\btheta$, the  Bayesian approach consists in finding the posterior \emph{distribution} of the skills $\pdf( \btheta | \un{y}_T )$. But, because finding the distributions is usually intractable, they are often assumed to belong to a particular parametrically defined family and the Gaussian approximation is often adopted
\begin{align}\label{gauss.static.posterior}
    \pdf( \btheta | \un{y}_T )\approx \mcN(\btheta; \hat{\bmu}, \hat{\bV}), 
\end{align}
where, to find $\hat{\bmu}$ we should calculate the mean from the posterior distribution $\pdf( \btheta | \un{y}_T )$. This also may be difficult, so we may prefer to set $\hat{\bmu}=\hat{\btheta}$, where $\hat\btheta$ is the mode of the distribution $\pdf( \btheta | \un{y}_T )\propto \Pr\{ \un{y}_T |\btheta \}\pdf(\btheta)$, and where $\pdf(\btheta)$ reflect a priori knowledge about $\btheta$. For the non-informative prior $\pdf(\btheta)$ the mode $\hat\btheta$ coincides with \eqref{ML.estimation}. 

While the problem of finding the posterior distribution of the skills is more general that finding the point estimate of the skills, with the model \eqref{gauss.static.posterior}, the mean $\hat\bmu$ may be treated as the \gls{map} point estimate, and the covariance $\hat\bV$ expresses the uncertainity of the \gls{map} estimation.

\subsection{Online rating}
The Bayesian approach to the online rating consists in finding the distribution of the skills $\btheta_t$ conditioned on the games' outcomes $\un{y}_t$, \ie 
\begin{align}\label{theta.aposteriori}
\pdf(\btheta_t|  \un{y}_{t} )&=
\pdf(\btheta_t|  \un{y}_{t-1}, y_t )
\propto 
\PR{y_t|\btheta_t}
\int  \pdf(\btheta_t , \btheta_{t-1}|  \un{y}_{t-1}, ) \dd \btheta_{t-1}\\
\label{theta.aposteriori.2}
&=\PR{y_t|\btheta_t}
\int  \pdf(\btheta_t|\btheta_{t-1})\pdf(\btheta_{t-1}|  \un{y}_{t-1}) \dd \btheta_{t-1},
\end{align}
where we exploited the Markovian property of \eqref{dumped.Markov}, \ie  the knowledge of $\btheta_{t-1}$ is sufficient to characterize the distribution of $\btheta_{t}$. 

The relationship  \eqref{theta.aposteriori.2} allows us to calculate the distribution $\pdf(\btheta_{t}|  \un{y}_{t})$ recursively, \ie from $\pdf(\btheta_{t-1}|  \un{y}_{t-1})$. This is what the online rating is actually about: as soon as the game outcomes become available, we estimate the (distribution of the) skills by exploiting the previously obtained estimation results. Such recursive calculation of the posterior distribution from \eqref{theta.aposteriori.2} has been already dealt with, \eg in \cite{Fahrmeir92}, \cite{Fahrmeir94} which also recognized that the formulation \eqref{theta.aposteriori.2} underlies the well known Kalman filtering \cite[Ch.~12-13]{Moon00_Book}. 

In order to make \eqref{theta.aposteriori.2} tractable, \cite{Fahrmeir92} (and many works that followed) rely on a Gaussian parametric representation of  $\pdf(\btheta_{t}|  \un{y}_{t})$, akin to \eqref{gauss.static.posterior}, \ie  $\tilde\pdf(\btheta_{t}|  \un{y}_{t})=\mcN(\btheta_{t}; \bmu_t, \bV_t)$ which allows us to implement the approximate version of \eqref{theta.aposteriori.2} as
\begin{align}\label{theta.aposteriori.tilde}
\hat \pdf(\btheta_{t}|  \un{y}_{t}) 
&=
\PR{y_t|\btheta_t}
\int  \pdf(\btheta_{t}|\btheta_{t-1})\tilde{\pdf}(\btheta_{t-1}|  \un{y}_{t-1}) \dd \btheta_{t-1},\\
\label{project.aposteriori}
\tilde \pdf(\btheta_{t}|  \un{y}_{t} )
&\propto
\mcP\left[\hat \pdf(\btheta_{t}|  \un{y}_{t}) \right],
\end{align}
where $\mcP[\pdf(\btheta)]$ is the operator projecting $\pdf(\btheta)$ on the space of Gaussian distribution, of which we consider the following possible forms with varying degree of simplification
\begin{align}\label{covariance.cases}
    \tilde\pdf(\btheta_t|\un{y}_t)=
    \begin{cases}
    \mcN(\btheta_t, \bmu_t, \bV_t) & \text{matrix-covariance model}\\
    \mcN(\btheta_t, \bmu_t, \tr{diag}(\bv_t)) & \text{vector-covariance model}\\
    \mcN(\btheta_t, \bmu_t, v_t\bI) & \text{scalar-covariance model}
    \end{cases},
\end{align}
where $\tr{diag}(\bv)$ is the diagonal matrix with diagonal elements gathered in the vector $\bv$. The vector- and scalar-covariance models are particularly suited for online rating when the number of players, $M$, is large, because we only need to estimate a vector/scalar instead of $M\times M$ covariance matrix. The vector-covariance model is the basis for the derivation of the TrueSkill and Glicko algorithms. 

The scalar-covariance model may be found in \cite{Ingram21} (with an additional assumption of the variance being constant) and is justified in the sports where players are uniformly scheduled to play throughout the season. Then, at any point of time, the number of games played is similar for all players so we can expect that the uncertainty (expressed by the variance) be similar for all the posterior estimates. On the other hand, in eSports, there may be significant differences between the number of games played by different players; in particular, the players who are new to the game should be characterised by a larger uncertainty then those who have been playing for a long time, and thus the scalar-covariance model may be inadequate.

The projection in \eqref{project.aposteriori} is done by finding $\tilde\pdf(\btheta_t|\un{y}_t)$ which minimizes the \gls{kl} distance to the projection argument $\hat\pdf(\btheta_t|\un{y}_t)$; this is done using the following.
\begin{proposition}\label{Prop:DKL}
The parameters of the distribution \eqref{covariance.cases},  $\tilde\pdf(\btheta_t|\un{y}_t)$, closest to  $\hat\pdf(\btheta_t|\un{y}_t)$ (in the sense of the \gls{kl} distance), should be set as follows:
\begin{align}
    \label{bmu.yt}
    \bmu_t&=\Ex[ \btheta_t |\un{y}_t]\\
    \label{bV.yt}
    \bV_t&=\Ex[(\btheta_t-\bmu_t)(\btheta_t-\bmu_t)\T |\un{y}_t]\\
    \label{bv.t.DKL}
    \bv_t&= \tnr{di}(\bV_t)\\
    \label{v.t.DKL}
    v_t & =\frac{1}{M}\bone\T\bv_t,
\end{align}
where $\tnr{di}(\bV)$ extracts the diagonal from the matrix $\bV$, and $v_t$ in \eqref{v.t.DKL} calculates the arithmetic average of the elements in $\bv_t$.
\begin{proof} 
\appref{Proof:DKL} 
\end{proof}
\end{proposition}

So, to implement the projection $\mcP[\cd]$, and irrespectively which covariance model in \eqref{covariance.cases} we decide to use, we must first calculate (exactly or approximately) the mean \eqref{bmu.yt} and the covariance \eqref{bV.yt} from the distribution $\pdf(\btheta_t|\un{y}_t)$. In the case of the vector (respectively, the scalar)-covariance model, we obtain the vector $\bv_t$ (respectively, the scalar $v_t$) from the covariance matrix,  via \eqref{bv.t.DKL} (respectively, via  \eqref{v.t.DKL}). 

The algorithm based on the  matrix-covariance model will be called a \gls{kf} rating and we show it in \secref{Sec:KF}. We will use it in \secref{Sec:SKF} to show how the \acrfull{skf} rating, based on the vector/scalar-covariance models may be obtained.

\subsection{Kalman filter}\label{Sec:KF}

The integral in \eqref{theta.aposteriori.tilde} is calculated from  \eqref{pdf.Normal} and \eqref{dumped.Markov} as\footnote{We use the relationship $\mcN(\btheta;\bmu_1,\bV_1)\mcN(\btheta; \bmu_2,\bV_2)=\mcN(\btheta;\bmu_3,\bV_3)\mcN(\bmu_1;\bmu_2;\bV_1+\bV_2)$ \cite[Ch.~8.4]{Barber12_Book}.}
\begin{align}\label{integral.prod.matrix}
\int  \pdf(\btheta_{t}|\btheta_{t-1})\tilde{\pdf}(\btheta_{t-1}|  \un{y}_{t-1}) \dd \btheta_{t-1}
&=
\mcN(  \btheta_{t} ; \beta_{t}\bmu_{t-1} , \ov\bV_{t}  ),
\end{align}
where 
\begin{align}\label{ov.V.t}
    \ov\bV_t &= \beta_{t}^2 \bV_{t-1} +\epsilon_t \bI
\end{align}
is the covariances matrix of the skills at time $t$ estimated from the observation $\un{y}_{t-1}$.

Using \eqref{integral.prod.matrix} and \eqref{pdf.y.theta} in \eqref{theta.aposteriori.tilde} yields
\begin{align}\label{pdf.n}
\hat\pdf(\btheta_{t}|  \un{y}_{t} ) &\propto \exp\big(Q(\btheta_{t})\big)\\
\label{Q.theta}
Q(\btheta)&=\ell(\bx_{t}\T\btheta/s;y_t)  -\frac{1}{2}(\btheta-\beta_{t}\bmu_{t-1})\T \ov\bV_t^{-1}(\btheta-\beta_{t}\bmu_{t-1})
\end{align}
and thus, by finding its mode
\begin{align}\label{mean.mode.Q}
    \bmu_t = \argmax_{\btheta} Q(\btheta)
\end{align}
and the inverse of the negated Hessian, $\bV_t=[-\nabla^2_{\btheta}Q(\btheta)]^{-1}|_{\btheta=\bmu_t}$, we obtain the approximate solution to the projection 
\begin{align}
   \mcP[\hat\pdf(\btheta_{t}|  \un{y}_{t} )]=\mcN(\btheta_t; \bmu_{t}, \bV_{t}).
\end{align}

We have to solve \eqref{mean.mode.Q} and this may be done 
by replacing $\ell(\bx_{t}\T\btheta/s;y_t)$ with a quadratic approximation
\begin{align}\label{Taylor.expansion}
\tilde\ell(\bx_{t}\T\btheta/s;y_t)\approx
\ell(\bx_{t}\T\btheta_{\tr{o}}/s;y_{t})+ [\nabla_{\btheta} \ell(\bx_{t}\T\btheta_{\tr{o}}/s;y_{t})]\T\big(\btheta-\btheta_{\tr{o}}\big) +\frac{1}{2}(\btheta-\btheta_{\tr{o}})\T\nabla_{\btheta}^2 \ell(\bx_{t}\T\btheta_{\tr{o}}/s;y_{t}) \big(\btheta-\btheta_{\tr{o}}\big),
\end{align}
obtained by developing $\ell(\bx_{t}\T\btheta;y_t)$ via Taylor series around  $\btheta_{\tr{o}}$, where the
gradient and the Hessian of $\ell(\bx_{t}\T\btheta;y_t)$ are calculated as
\begin{align}\label{grad.ell}
\nabla_{\btheta} \ell(\bx_{t}\T\btheta;y_t) &= \frac{1}{s}g(\bx_{t}\T\btheta/s;y_{t})\bx_{t}\\
\nabla_{\btheta}^2 \ell(\bx_{t}\T\btheta;y_t) &= -\frac{1}{s^2}h(\bx_{t}\T\btheta/s;y_{t})\bx_{t}\bx_{t}\T,
\end{align}
with the  first- and the second derivatives of the scalar function $\ell(z;y_{t})$ denoted as
\begin{align}
 g(z;y_{t})&=\frac{\dd}{\dd z}\ell(z;y_{t}),\\
 h(z;y_{t})&=- \frac{\dd^2}{\dd z^2}\ell(z;y_{t});
\end{align}
we note that $h(z;y_t)>0$ because $\ell(z;y_{t})$ is concave in $z$.

Replacing $\ell(\cd;y_{t})$ with $\tilde\ell(\cd;y_{t})$ in \eqref{Q.theta}, the mode \eqref{mean.mode.Q} is obtained when the gradient of $Q(\btheta)$ goes to zero, \ie
\begin{align}
\nabla_{\btheta} Q(\btheta)|_{\btheta=\bmu_{t}}\approx \frac{1}{s}g(\bx_{t}\T\btheta_{\tr{o}}/s;y_{t})\bx_{t}
-\frac{1}{s^2}h(\bx_{t}\T\btheta_{\tr{o}}/s;y_{t})\bx_{t}\bx_{t}\T\big(\bmu_{t}-\btheta_{\tr{o}}\big) - \ov\bV_t^{-1}\big(\bmu_{t}-\beta_{t}\bmu_{t-1}\big) =\bzero,
\end{align}
which is solved by
\begin{align}\label{theta0.prime}
\bmu_{t} &= \bV_t\Big[ \bx_{t} \big(\frac{1}{s}g(\bx_{t}\T\btheta_{\tr{o}}/s; y_{t}) + \frac{1}{s^2}h(\bx_{t}\T\btheta_{\tr{o}}/s;y_{t})\bx_{t}\T\btheta_{\tr{o}}\big) +\ov\bV_{t}^{-1}\beta_{t}\bmu_{t-1} \Big],
\end{align}
where
\begin{align}
\bV_t & = \big[\frac{1}{s^2}h(\bx_{t}\T\btheta_\tr{o}/s;y_{t})\bx_{t}\bx_{t}\T + \ov\bV_t^{-1}\big]^{-1}\\
\label{eq:bV.full}
&=
\ov\bV_t -\ov\bV_t\bx_{t}\bx\T_{t}\ov\bV_t
\frac
{ h(\bx_{t}\T\btheta_\tr{o}/s;y_{t})}
{s^2+h(\bx_{t}\T\btheta_\tr{o}/s;y_{t})\omega_t},
\end{align}
$\omega_t=\bx\T_{t}\ov\bV_t\bx_{t}$,  
and \eqref{eq:bV.full} is obtained via matrix inversion lemma \cite[Sec.~4.11]{Moon00_Book}.

Combining \eqref{eq:bV.full} with \eqref{theta0.prime} yields
\begin{align}\label{btheta.full}
\bmu_{t}&=\beta_{t}\bmu_{t-1} + \ov\bV_t\bx_t\frac{sg(\bx_{t}\T\btheta_{\tr{o}}/s;y_{t})+h(\bx_{t}\T\btheta_{\tr{o}}/s;y_{t})\bx_{t}\T(\btheta_{\tr{o}}-\beta_{t}\bmu_{t-1})}{s^2+h(\bx_{t}\T\btheta_{\tr{o}}/s;y_{t})\omega_t}.
\end{align}
After this first update, a further refinement may be  obtained by alternating between \eqref{btheta.full} and the reassignment $\btheta_{\tr{o}}\leftarrow\bmu_{t}$ but, of course, it is much easier to use just one iteration with  $\btheta_{\tr{o}}=\beta_{t}\bmu_{t-1}$, which yields a simple update of the skills' mean and covariance matrix, and that defines the \gls{kf} rating:
\begin{empheq}[box=\fbox]{align}
\label{ov.bV.update.KF}
\ov{\bV}_t&\leftarrow \beta_{t}^2 \bV_{t-1}+\epsilon_t \bI\\
\omega_t &\leftarrow \bx_t\T\ov{\bV}_t \bx_t\\
g_t &\leftarrow g(\beta_{t}\bx_{t}\T\bmu_{t-1}/s;y_{t})\\
h_t &\leftarrow h(\beta_{t}\bx_{t}\T\bmu_{t-1}/s;y_{t})\\
\label{oneshot.mean.KF}
\bmu_{t}&\leftarrow\beta_{t}\bmu_{t-1} + \ov\bV_t\bx_t\frac{sg_t}{s^2+h_t\omega_t}\\
\label{ht.update.KF}
h_t &\leftarrow h(\beta_{t}\bx_{t}\T\bmu_{t}/s;y_{t})\\
\label{Vt.update.KF}
\bV_{t}
&\leftarrow
\ov\bV_t -\ov\bV_t\bx_{t}\bx\T_{t}\ov\bV_t
\frac
{ h_t}
{s^2+h_t\omega_t},
\end{empheq}
where \eqref{Vt.update.KF} is obtained from \eqref{eq:bV.full} by setting $\btheta_\tr{o}\leftarrow \beta_{t} \bmu_t$. 
On the other hand, since we are in the realm of approximations, we might also use $\btheta_{\tr{o}}=\beta_{t} \bmu_{t-1}$ in \eqref{eq:bV.full}, which amounts to ignoring/removing \eqref{ht.update.KF}; this is what we do in the rest of this work.

The initialization is done by $\bV_0 \leftarrow v_0 \bI$, where $v_0$ is the prior variances of the skills.

\subsection{Simplified Kalman Filters and Stochastic Gradient}\label{Sec:SKF}


We can now translate \eqref{ov.bV.update.KF}-\eqref{Vt.update.KF} taking into account the fact that the matrices are diagonal, \ie by replacing $\bV_t$ with $\tr{diag}(\bv_t)$; this yields the following equations of the \gls{vskf} rating:
\begin{empheq}[box=\fbox]{align}
\label{ov.bV.update.vSKF}
\ov{\bv}_t&\leftarrow \beta_{t}^2 \bv_{t-1}+\epsilon_t \bone\\
\omega_t&\leftarrow\sum_{m\in\set{\mcI_{t},\mcJ_{t}}} \ov{v}_{t,m}\\
g_t &\leftarrow g(\beta_{t}\bx_{t}\T\bmu_{t-1}/s;y_{t})\\
h_t &\leftarrow h(\beta_{t}\bx_{t}\T\bmu_{t-1}/s;y_{t})\\
\label{oneshot.mean.vSKF}
\bmu_{t}&\leftarrow\beta_{t}\bmu_{t-1} + \ov\bv_t\odot\bx_t\frac{s g_t}{s^2+h_t\omega_t}\\
\label{Vt.update.vSKF}
\bv_{t}
&\leftarrow
\ov\bv_t \odot\Big( \bone  -  \ov\bv_t \odot |\bx_t|\frac
{ h_t}
{s^2+h_t\omega_t}\Big),
\end{empheq}
where $\odot$ denotes the element-by-element multiplication, and the initialization is done as $\bv_0\leftarrow v_0 \bone$.

In particular, exploiting the form of the scheduling vector $\bx_t$ (with elements $x_{t,m}\in\set{-1,0,1}$), we 
see that the  players who are not involved in the game ($m\notin\set{\mcI_t,\mcJ_t}$ and  thus $x_{t,m}=0$), are updated as
\begin{align}
\mu_{t,m}&=\beta_{t}\mu_{t-1,m},\\
v_{t,m}&=\beta_{t}^2 v_{t-1,m} +\epsilon_t.
\end{align}
Most often $\beta_{t}=1$ will be used and then the means of the skills do not change, but the variance grows with $t$. This is compatible with the intuition we have about the rating procedure: the players not involved in the game should not change their mean (remember, the mean is approximated by the mode, thus it should be interpreted as the \gls{ml} estimate of the skill), while the growing variance corresponds to increased uncertainty about the skills' values due to passed time. 

It is worthwhile to note that a similar algorithm was proposed in \cite{Paleologu13} for $\ell(z;y)$ being a quadratic function, \ie in the context in which  the Kalman algorithm is conventionally used.

As for the scalar-covariance model, we have to replace $\bv_t$ with $v_t\bone$ in the \gls{vskf} rating, which will yield the following equations of the \gls{sskf} rating:
\begin{empheq}[box=\fbox]{align}
\ov{v}_t&\leftarrow \beta_{t}^2 v_{t-1}+\epsilon_t\\
\omega_t&\leftarrow 2F \ov{v}_{t}\\
g_t &\leftarrow g(\beta_{t}\bx_{t}\T\bmu_{t-1}/s;y_{t})\\
h_t &\leftarrow h(\beta_{t}\bx_{t}\T\bmu_{t-1}/s;y_{t})\\
\label{oneshot.mean.sSKF}
\bmu_{t}&\leftarrow\beta_{t}\bmu_{t-1} + \ov{v}_t\bx_t\frac{s g_t}{s^2+h_t\omega_t}
\\
\label{Vt.update.sSKF}
v_{t}
&\leftarrow
\ov{v}_t \Big( 1  -   \frac{\omega_t}{M}\frac
{h_t}
{s^2+h_t\omega_t}\Big),
\end{empheq}
where $F=|\mcI_{t}|=|\mcJ_{t}|$ and the initialization requires setting $v_0$.

Another simplification is obtained if we assume that the variance $\ov{v}_t$ is constant across time $t$, \ie $\ov{v}_t=\ov{v}$, as done also in \cite{Ingram21}. We obtain then the \gls{fskf}
\begin{empheq}[box=\fbox]{align}
g_t &\leftarrow g(\beta_{t}\bx_{t}\T\bmu_{t-1}/s;y_{t})\\
h_t &\leftarrow h(\beta_{t}\bx_{t}\T\bmu_{t-1}/s;y_{t})\\
\label{oneshot.mean.fSKF}
\bmu_{t}&\leftarrow\beta_{t}\bmu_{t-1} + \ov{v}\bx_t\frac{sg_t }{s^2+h_t 2F\ov{v}},
\end{empheq}
where the initialization requires setting $\ov{v}$.

All the \gls{skf} rating algorithm adjust the mean in the direction of the gradient, $g_t$, of the log-likelihood $\ell(z_t/s;y_t)$; they differ in the way the adjustment step is calculated from the previous results which mostly depends on the the second-order derivative, $h_t$. 

And finally, ignoring $h_t$, \eqref{oneshot.mean.fSKF} may be written as 
\begin{empheq}[box=\fbox]{align}\label{SG.update}
    \bmu_{t}\leftarrow \bmu_{t-1} + \ov{v}/s \bx_t g_t,
\end{empheq}
which is the same as the \gls{sg} algorithm with the adaptation step being proportional to $\ov{v}$, the latter has the meaning of the posterior variance of the skills (which we suppose to be known).

As we will see, depending on the model, ignoring $h_t$ may make sense. In particular, using the Bradley-Terry or the Davidson models, we obtain $\lim_{|z|\rightarrow{\infty}}h(z;y_t) =0$, see \eqref{h.Logistic} and \eqref{Davidson.h}. That is, for large differences between the skills, the second derivative of $\ell(z;y_t)$, may indeed be close to zero.


At this point it is useful to comment on the use of the scale. While in practice, $s=400$, \eg \cite{fide_calculator}, \cite{eloratings.net}, \cite{Silver20} or $s=600$ \cite{fifa_rating} were applied, the value of $s$ is entirely arbitrary and, actually, irrelevant from the algorithmic point of view, as stated in the following:
\begin{proposition}\label{Prop:SKF.scale}
We denote by $\bmu_t(s,v_0,\epsilon)$ and by $\bV_t(s,v_0,\epsilon)$ the mean and the covariance matrix of the skills obtained using the \gls{kf} algorithm with the scale $s$, and initialization parameters $v_0$ and $\epsilon$. Then
\begin{align}
    \bmu_t(s,s^2v_0,s^2\epsilon)&=s\bmu_t(1,v_0,\epsilon)\\
    \bV_t(s,s^2v_0,s^2\epsilon)&=s^2\bV_t(1,v_0,\epsilon).
\end{align}

For the \gls{vskf} algorithm we will obtain $\bv_t(s,s^2v_0,s^2\epsilon)=s^2\bv_t(1,v_0,\epsilon)$ while for the \gls{sskf} algorithm,  $v_t(s,s^2v_0,s^2\epsilon)=s^2v_t(1,v_0,\epsilon)$, where $\bv_t$ and $v_t$ are shown to depend on $v_0$ and $\epsilon$. On the other hand, for the \gls{fskf} and the \gls{sg}, the mean depends solely on $\ov{v}$ and thus we  obtain $\bmu_t(s,s^2 \ov{v})=s\bmu_t(1,\ov{v})$.
\end{proposition}
\begin{proof} 
See \appref{Proof:SKF.scale}.
\end{proof}

\propref{Prop:SKF.scale} simply says that the scale, $s$, is not identifiable from the data so we can ignore it, \eg use $s=1$ (which simplifies the notation) and adjust only the parameters $\beta$,  $\epsilon$, and $v_0$. The scale may be then included in the final results by multiplying the means $\bmu_t$ (by $s$) and the co-/variances $\bV_t$, $\bv_t$, or $v_t$ (by $s^2$). 

Nevertheless, by introducing the scale we are able to compare our rating algorithms with those that can be found in the literature.

In particular, we can rewrite \eqref{SG.update}
\begin{align}\label{SG.update.s1}
    \bmu_{t}\leftarrow \bmu_{t-1} + K s \bx_t g_t,
\end{align}
where we use $\ov{v}=Ks^2$ with $K$ being the adaptation step defined for the scale $s=1$. Since $K$ should be seen as the variance $\ov{v}$, it clarifies the well-known variable-step strategy in the \gls{sg} adaptation, where the step $K$ is decreased after many games are played: this is when the posterior variance decreases.

\section{From skills-outcome models to new online ratings}\label{Sec:New.Ratings}

We will turn to the popular skills-outcome models that has been often used and find the functions $g(z;y_t)$ and $h(z;y_t)$ which must be used in the \gls{kf} and the \gls{skf} algorithms
\begin{itemize}
    \item Thurston model \cite{Thurston27} (binary games) uses $y_t=0$ for the away win and $y_t=1$ for the home win:
    \begin{align}
    \label{Thurston.L}
     L(z;y_{t})  &= \Phi\left(z\right)\IND{y_{t}=1}+\Phi\left(-z\right)\IND{y_{t}=0},\\
     \label{Thurston.g}
     g(z;y_{t}) &= V(z)\IND{y_t=1} -V(-z)\IND{y_t=0},\\
     \label{Thurston.h}
     h(z;y_{t})& =W(z)\IND{y_t=1} +W(-z)\IND{y_t=0},
     \end{align}
     where $\Phi(z)=\int_{-\infty}^z\ov\mcN(t)\dd t$, $\ov\mcN(t)=\mcN(t;0,1)$, and
     \begin{align}
        \label{Thurston.V}
        V(z)&=\frac{\ov{\mcN}(z)}{ \Phi\big( z\big)}, \\
        \label{Thurston.W}
        W(z)&=-V'(z)=V(z)\big(z+V(z)\big).
    \end{align}
    
    \item Bradley-Terry model \cite{Bradley52} (binary games), with $y_t=0$ (away win) and $y_t=1$ (home win)
    \begin{align}\label{L.Logistic}
    L(z;y_{t})&=F_\tr{L}\big( z\big)\IND{y_{t}=1}+F_\tr{L}\big( -z\big)\IND{y_{t}=0},\\
    \label{g.Logistic}
    g(z;y_{t}) &= \ln 10 \big( y_t-F_\tr{L}(z) \big),\\
    \label{h.Logistic}
    h(z;y_{t}) &= \left(\ln 10\right)^2 F_\tr{L}(z)F_\tr{L}(-z),
    \end{align}
    where we use the logistic function
    \begin{align}\label{F.Logistic}
        F_\tr{L}(z)=\frac{1}{1+10^{-z}}.
    \end{align}
    
    \item Davidson draw model \cite{Davidson70}, \cite{Szczecinski20}  with $y_t=0$ (away win),  and $y_t=1$ (draw), and $y_t=2$ (home win)
    \begin{align}
    \label{Davidson.L}
     L(z;y_{t})&=F_\tr{D}(-z)\IND{y_{t}=0}+\kappa\sqrt{F_\tr{D}(-z)F_\tr{D}(z)}\IND{y_{t}=1}+F_\tr{D}(z)\IND{y_{t}=2},\\
     \label{Davidson.g}
     g(z;y_{t}) &= 2\ln 10 \big( \hat{y}_t-G_\tr{D}(z) \big)\\
     \label{Davidson.h}
     h(z;y_{t})& =\left(\ln 10\right)^2\frac{\kappa 10^{z}+4 +\kappa 10^{-z}}{(10^{z}+\kappa+ 10^{-z})^2},
     \end{align}
     where $\hat{y}_t=\frac{1}{2}y$ may be treated as the ``score'' of the game, and
     \begin{align}
         F_\tr{D}(z)&=\frac{10^z}{10^{-z}+\kappa+10^{z}},\\
         \label{G.D.z}
         G_\tr{D}(z)&=\frac{10^{z}+\kappa/2}{10^{-z}+\kappa+10^{z}}.
     \end{align}
    
    Note that setting $\kappa=0$, \ie removing the possibility of draws, we obtain $F_\tr{L}(z) = G_\tr{D}(z/2)$, \ie we recover the equations of the Bradley-Terry model with halved scale. A simple, but less obvious observation is that, setting $\kappa=2$, we obtain $G_\tr{D}(z)=F_\tr{L}(z)$ and thus $g(z;y_t)$ in \eqref{g.Logistic} is half of  \eqref{Davidson.g}. 
    A direct consequence, observed in  \cite{Szczecinski20}, is that, even if the Bradley-Terry and the Davidson models are different, their \gls{sg} updates \eqref{SG.update} may be identical.
\end{itemize}

\subsection{Comparison with TrueSkill algorithm}\label{Sec:TrueSkill}

The TrueSkill algorithms is derived assuming that, for given skills $\btheta_t$, the outcome, $y_t$, is obtained by discretization of a variable $d_t=z_t +u_t=\bx\T_t\btheta_t+u_t$, \ie
\begin{align}\label{y_t.TrueSkill}
y_t=\IND{d_t \ge 0},
\end{align}
where $u_t$ is a zero-mean Gaussian variance with variance $\sigma^2$. Thus, $\PR{y_t=y|z_t}=L(z_t/\sigma;y)$, where $L(\cd,\cd)$ is given by the Thurston equation \eqref{Thurston.L}. So while $\sigma^2$ is the variance of the variable $u_t$, we may also treat $\sigma$ as the scale in the Thurston model.\footnote{To be more precise, the variance $\sigma^2$ in the TrueSkill algorithm is proportional to the number of players in the team, $F$.}

Considering the binary games, the  TrueSkill algorithm, described in \cite{trueskill} and \cite{Herbrich06} for two players, may be summarized as follows (for $m\in\set{i_t,j_t}$):
\begin{align}
\ov{\bv}_t &\leftarrow\bv_{t-1}+\epsilon \bone,\\
\omega_t  &\leftarrow \ov{v}_{t,i_t}+\ov{v}_{t,j_t},\\ 
\tilde{\sigma}_t&\leftarrow \sigma \sqrt{1+\omega_t/\sigma^2},\\
\tilde{g}_t &\leftarrow g(\bx_{t}\T\bmu_{t-1}/\tilde{\sigma}_t; y_t),\\
\tilde{h}_t &\leftarrow h(\bx_{t}\T\bmu_{t-1}/\tilde{\sigma}_t; y_t),\\
\label{update.theta.t.TrueSkill}
\mu_{t,m} &\leftarrow\mu_{t-1,m} + x_{t,m}\ov{v}_{t,m}  \frac{ \tilde{g}_t\sigma}{\sigma^2\sqrt{1+\omega_t/\sigma^2}}, \\
\label{update.var.t.TrueSkill}
v_{t,m} & \leftarrow\ov{v}_{t ,m} \Big(1 - \frac{\ov{v}_{t,m}\tilde{h}_t}{\sigma^2+\omega_t}\Big),
\end{align}
where $h(\cd;\cd)$ and $g(\cd;\cd)$ are derived in \eqref{Thurston.L}-\eqref{Thurston.W} for the Thurston model.

The differences with the \gls{vskf} algorithm are the following: i) the scale used to calculate the fist and the second derivatives is increased by the factor $\sqrt{1+\omega_t/\sigma^2}$, and ii) the denominator of the update terms in \eqref{update.theta.t.TrueSkill} and \eqref{update.var.t.TrueSkill} is not affected by $h_t$ as it is the case in the corresponding equations \eqref{oneshot.mean.vSKF} and \eqref{Vt.update.vSKF} of the \gls{vskf} algorithm. In particular, knowing that $\omega_t h_t\leq \omega_t$, we see that the posterior variance $v_{t,m}$ decreases faster in the \gls{vskf} algorithm than it does in the TrueSkill algorithm.

Numerical examples shown in \secref{Sec:Num.results} will allow us to asses the impact of these differences between the algorithms.

\subsection{Comparison with Glicko algorithm}\label{Sec:Glicko}

The Glicko algorithm, defined for two players in \cite[Eqs.~(9)-(10)]{Glickman99}  may be formulated using our notation as follows (for $m\in\set{i_t,j_t}$):
\begin{align}
\ov{\bv}_t&\leftarrow \bv_{t-1} +\epsilon_t\bone,\\
\omega_t&\leftarrow \ov{v}_{t,i_t}+\ov{v}_{t,j_t},\\
\label{tilde.sigma.Glicko}
\tilde{\sigma}_{t,m}&\leftarrow \sigma r( \omega_t-\ov{v}_{t,m}),\\
\tilde{g}_{t,m}&\leftarrow 
g\big(\bx_{t}\T\bmu_{t-1}/\tilde{\sigma}_{t,m}; y_t )\\
\tilde{h}_{t,m}&\leftarrow h\big(\bx_{t}\T\bmu_{t-1}/\tilde{\sigma}_{t,m}; y_t\big)\\
\label{update.mu.t.Glicko}
\mu_{t,m}&\leftarrow\mu_{t-1,m} +  
\ov{v}_{t,m} x_{t,m} \frac{\tilde{\sigma}_{t,m} \tilde{g}_{t,m}}{\tilde{\sigma}_{t,m}^2+\ov{v}_{t,m} \tilde{h}_{t,m}},\\
\label{update.v.t.Glicko}
v_{t,m} & \leftarrow \ov{v}_{t,m}\frac{\tilde{\sigma}_{t,m}^2}{\tilde{\sigma}_{t,m}^2+\ov{v}_{t,m} \tilde{h}_{t,m}},
\end{align}
where 
\begin{align}
\label{rtjt}
r(v)&=\sqrt{1+  \frac{v a}{\sigma^2}},
\end{align}
$a=3\ln^2 10/\pi^2$ is the factor which allows us to approximate the logistic distribution with the Gaussian distributions (see discussion in \secref{Sec:Synthetic}), and 
$g(z;y_t)$ and $h(z;y_t)$ are defined for the Bradley-Terry model, respectively, in \eqref{g.Logistic} and \eqref{h.Logistic}.

The difference between $g_{t}$, $h_t$ in the \gls{vskf} rating (based on the Bradley-Terry model) and $\tilde{g}_{t,m}$, $\tilde{h}_{t,m}$ in the Glicko algorithm, is due to the presence of the factor $r(\omega_t-\ov{v}_{t,m})$ which multiplies the scale in \eqref{tilde.sigma.Glicko}. However, this factor tends to unity when the variance of the opposing players decreases, as it is the case after convergence. Then, we may use $\tilde{g}_{t,m}\approx g_t$ and $\tilde{h}_{t,m}\approx h_t$.

Further, if we use $\ov{v}_{t,m}$ instead of $\omega_t$ in the denominator of the mean \eqref{oneshot.mean.vSKF} and of the variance \eqref{Vt.update.vSKF} updates in the \gls{vskf} algorithm, we will obtain, respectively, the Glicko updates of the mean, \eqref{update.mu.t.Glicko} and of the variance, \eqref{update.v.t.Glicko}. But, because $\ov{v}_{t,m}<\omega_t$, the update step size is always larger in the Glicko algorithm comparing to the \gls{vskf} algorithm.

To asses the impact of the above differences on the performance we will evaluate the Glicko algorithm using numerical examples in \secref{Sec:Num.results}.

\subsection{Comparison with Elo algorithm}\label{Sec:Elo}

Considering again the binary games and using the Bradley-Terry model, the \gls{sg} update \eqref{SG.update.s1} may be written as
\begin{align}\label{SG.BT}
    \bmu_t \leftarrow \bmu_{t-1} + \tilde{K} s \bx_t \big(y_t-F_\tr{L}(z_t)\big),
\end{align}
where $\tilde{K}$ absorbs the  term $\ln 10$ from \eqref{g.Logistic}, and we recognize \eqref{SG.BT} as the well-known Elo rating algorithm. The fact that the Elo algorithm may be seen as the \gls{sg} update in the Bradly-Terry model has been already noted before, \eg in \cite{Kiraly17}, \cite{Szczecinski20}, \cite{Lasek20}.

On the other hand, using the Thurston model and after simple algebraic transformations of \eqref{Thurston.g} we obtain the following \gls{sg} update:
\begin{align}\label{SG.Thurston}
    \bmu_t \leftarrow \bmu_{t-1} + K s \bx_t \big(y_t-\Phi(z_t)\big)\xi(z_t),
\end{align}
where $\xi(z)=\ov{\mcN}(z)/\big[\Phi(z)\Phi(-z)]\big]$. 

Since $\xi(z)$ is not constant, \ie it depends on $z$, \eqref{SG.Thurston}
is \emph{not the same} as the Elo rating algorithm proposed initially by \cite{Elo08_Book} under the following form:
\begin{align}\label{Elo.original}
    \bmu_t \leftarrow \bmu_{t-1} + K s \bx_t \big(y_t-\Phi(z_t)\big).
\end{align}

In other words, the original version of the Elo algorithm \eqref{Elo.original} does not implement the \gls{sg} update in the Thurston model. We indicate it merely for completeness of the analysis because, nowadays, the Elo algorithm is practically always used with the Bradley-Terry model as defined in \eqref{SG.BT}.

\section{Numerical examples}\label{Sec:Num.results}

We will proceed in two steps. First, in order to assess the effect of approximations, we will use the synthetic data generated using the predefined skills-outcome models and  the Gaussian random walk for skills dynamics defined  in \secref{Sec:Model}. In this way, knowing exactly the model underlying the data and using it for the derivation of the algorithm, the eventual differences between the algorithms will be due to the approximations. Further the effect of the model mismatch may be also assessed using the algorithms based on the model different from the one used to generate the data.

The insight obtained from the synthetic examples will allow us to interpret the results obtained from empirical data.

\subsection{Synthetic data}\label{Sec:Synthetic}
We suppose there are $M$ players in the pool and every ``day" (or any other time unit) there are $J=M/2$ games with random scheduling; the season lasts $D$ days. The time dependence required by \eqref{epsilon.t} is defined as $\tau(1)=\tau(2)=\ld=\tau(J)=0$, $\tau(J+1)=\tau(J+2)=\ld=\tau(2J)=1$ etc. The number of games in the season is equal to  $T=DJ$. We use $M=20$ ($J=10$) and $D=100$, thus $T=1000$.\footnote{This bears a resemblance to a ``typical'' football season where, on average each team plays once per week. In practice, of course, the number of weeks, $D$, cannot be too large, \eg $D<40$.}

To generate the sequence of skills $\btheta_t, t=1, \ld, T$ we draw $\theta_{0,m}$ from a zero-mean, unit-variance Gaussian distribution; the remaining skills are obtaied using \eqref{dumped.Markov} with $\hat{\beta}=0.998$ and $\hat{\epsilon}=1-\hat{\beta}^2$. In this way we guarantee that $\Ex[\theta_{t,m}]=0$ and $\Ex[\theta^2_{t,m}]=1$. To evaluate how the perturbation of the skills affects the algorithms, after the day $d_{\tr{switch}}=40$, we remove the first $m_{\tr{switch}}=5$ players, which are already in the game, and replace them with new players whose skills $\theta_{d_{\tr{switch}},m}$ are generated from zero-mean, unit-variance Gaussian distribution; for $d>d_{\tr{switch}}$ we use again the random walk \eqref{dumped.Markov}. Such a ``switch'' scenario loosely reflects the case of new players joining the online games\footnote{The analogy is admittedly of limited scope because the online games do not care about the total number of players, $M$, to be constant. On the other hand, we do care, because we want to be able apply the \gls{kf} rating as defined in \secref{Sec:KF}} andallows us to evaluate how the algorithms deal with abrupt changes of the skills.

The algorithms adjust to this ``switch'' by zeroing the means and adjusting the variances of the newly arrived players, that is, setting $\mu_{t-1,m}\leftarrow0, v_{t-1,m}\leftarrow v_0, m=1,\ld,m_{\tr{switch}}$, where $t=M d_{\tr{switch}}/2$. For the \gls{kf} rating we also have to zero the covariances, $V_{t-1,m,l}\leftarrow 0, m=1,\ld,m_{\tr{switch}}, \forall l$, while in the \gls{sskf} rating we recalculate the average variance as $v_{t-1}\leftarrow v_{t-1}+ (v_0 - v_{t-1})m_{\tr{switch}}/M$.

The results of binary games, $y_t$, are generated with the probability defined by the Thurston model, \ie the probability of the home win is defined by
\begin{align}\label{p.t}
   p_t=\PR{y_t=1} = \Phi(\bx_{t}\T\btheta_t/\sigma),
\end{align}
where $\sigma^2$ may be interpreted as the variance of the Gaussian noise added to difference between the skills before the discretization defined in  \eqref{y_t.TrueSkill}. Thus, increasing $\sigma$ the observations becomex more ``noisy''. Most of the results are shown for $\sigma=1$ and later we will asses the impact of larger $\sigma$.

The performance of the algorithm is measured by the \gls{kl} divergence between the actual distribution of the games outcomes (defined by $p_t$) and the estimated distribution (defined by $L(\beta_t\bx_{t}\T\bmu_{t-1}/s; 1)$)
\begin{align}\label{DKL.metric}
    \tr{D}_{t} = p_t\log\frac{p_t}{L(\beta_t\bx_{t}\T\bmu_{t-1}/s; 1)} + (1-p_t)\log\frac{1-p_t}{L(\beta_t\bx_{t}\T\bmu_{t-1}/s; 0)}
\end{align}
that can be evaluated here because we know how data is generated.

\begin{figure}[bt]
\psfrag{beta0.98}{$\beta=0.98$}
\psfrag{xlabel}{$d$}
\begin{tabular}{cc}
    \scalebox{\sizf}{\includegraphics[width=\sizfs\linewidth]{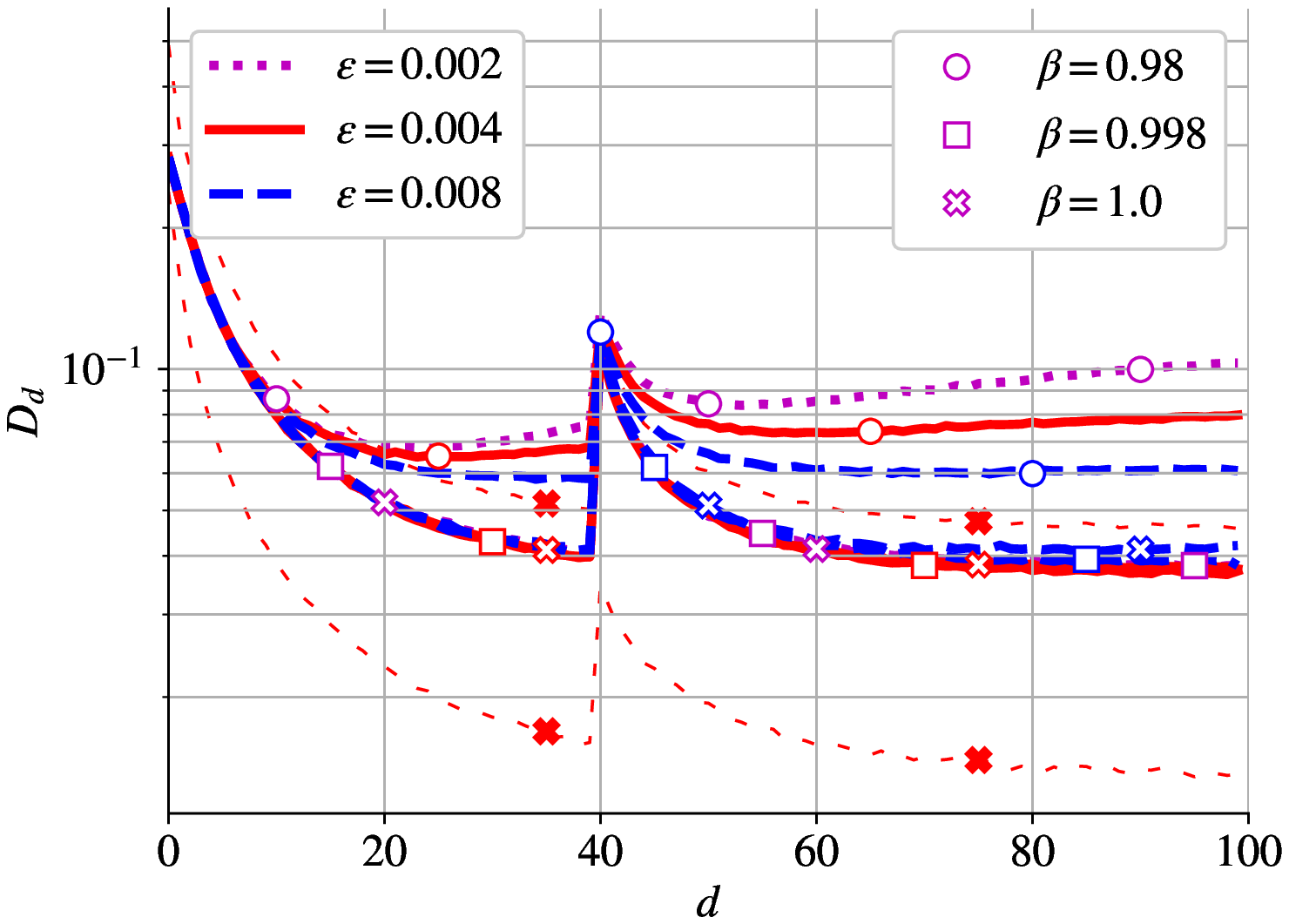}} &
    \scalebox{\sizf}{\includegraphics[width=\sizfs\linewidth]{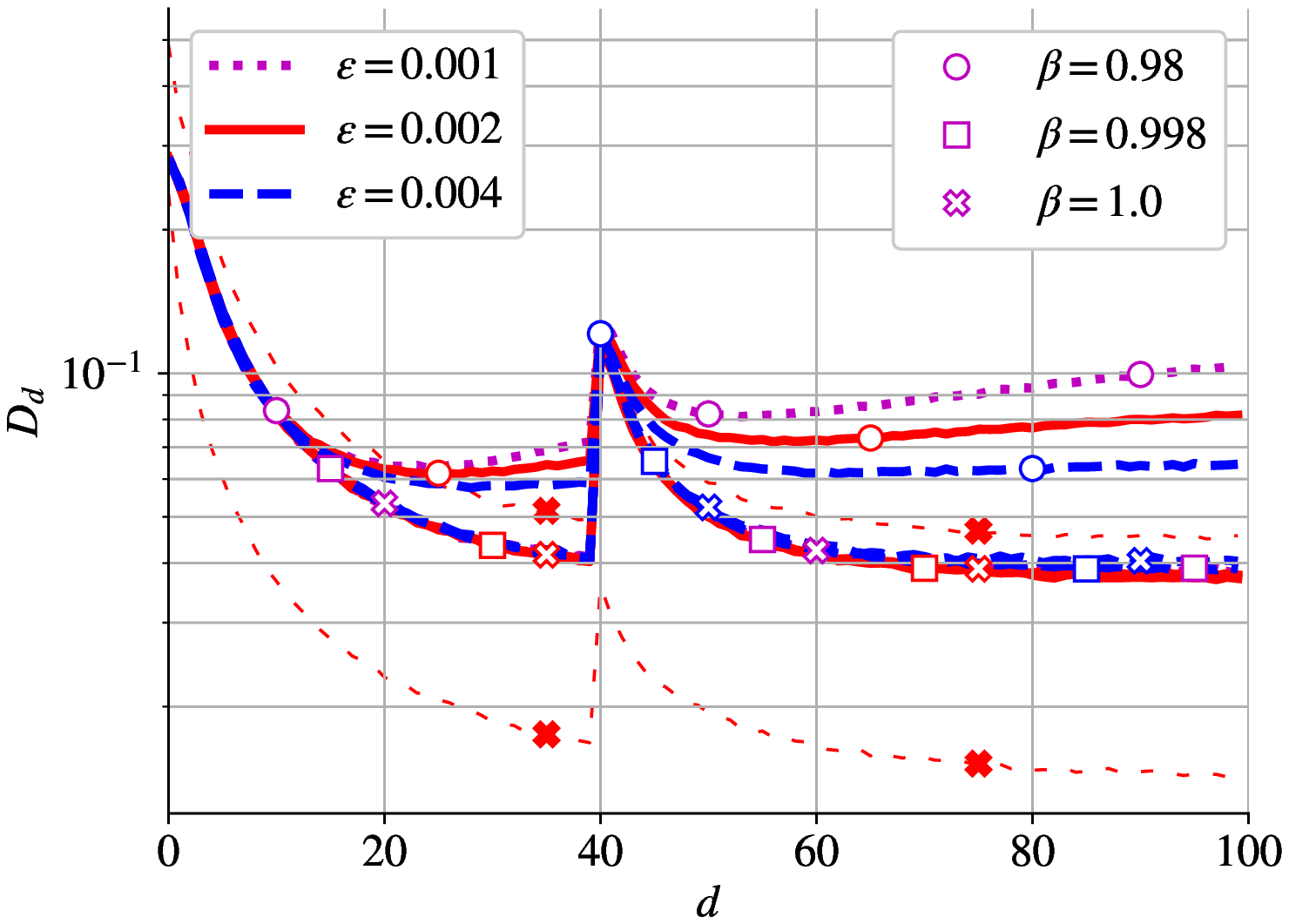}}\\
a) & b)
\end{tabular}
\caption{Average \gls{kl} divergence for different values of $\beta$ and $\epsilon$ used in the \gls{kf} rating algorithm based on a) the Thurston model and b) the Bradley-Terry model. The loosely dashed lines indicate the median and the third quartile (only for $\beta=1$ and a) $\epsilon=0.004$, b) $\epsilon=0.002$).}\label{Fig:start}
\end{figure}

Of course, $\tr{D}_t$ obtained from randomly generated data is also random, so we show in \figref{Fig:start} its mean obtained from 5000 simulation runs of the \gls{kf} rating algorithm based on the Thurston as well as on the Bradley-Terry models, for different values of $\beta$ and $\epsilon$, with $s=\sigma$, and $v_0=1$. Note that, for the Thurston model, using $v_0=1$ and $s=\sigma$, the same model is used for the data generation and for the rating.

To smooth the results, we show the average of all the results obtained in the same day $d$. We observe that the performance of the algorithm depends on $\beta$ being close to the actual value in the data-generation model: once $\beta$ is suitably  chosen, the effect of $\epsilon$ is of lesser importance. Nonetheless, the best mean performance is obtained with $\epsilon=\hat{\epsilon}=0.004$ (for the Thurston model) and $\epsilon=0.002$ (for the Bradley-Terry model). These parameters will be also used in other algorithms.

To put this (rather limited) importance of $\epsilon$ into perspective, we also show in \figref{Fig:start} the median and the third quartile (loosely dashed lines, for $\beta=1$ and the optimal value of $\epsilon$): it indicates that there is more variability due to the randomness of the metric $\tr{D}_{d}$ than due to the change in $\epsilon$.

To answer the question why, using two different models (Thurston and Bradley-Terry) practically the same results are obtained in \figref{Fig:start}a and \figref{Fig:start}b, we first note that the logistic distribution (underlying the Bradley-Terry model with the scale $s_{\tr{L}}$) has the variance equal to $s^2_{\tr{L}}/a$ (where $a\approx 1.6$, is defined after \eqref{rtjt}), while the Gaussian distribution (underpinning the Thurston model with the scale $s_{\tr{G}}$) has the variance $s^2_{\tr{G}}$. By equalizing the second moments of both distributions we obtain the relationship $s_{\tr{L}}= s_{\tr{G}}\sqrt{a} \approx 1.3 s_{\tr{G}}$. In other words, the Thurston model may be approximated with the Bradley-Terry model if we increase the scale by $\sqrt{a}$. 

On the other hand, from \propref{Prop:SKF.scale} we know that we may turn the table: we might keep the scale $s_{\tr{L}}=1$ and then multiply the parameters $v_0$ and $\epsilon$ by the factor $1/\sqrt{a} \approx 0.78$; since it is close to one, the effect of using the same scale and parameters in different models is merely visible. Nonetheless, using a value of $\epsilon\approx 0.5\hat{\epsilon}$ improves (slightly) the performance. The only remaining element which should be adjusted is the initial uncertainty defined by the variance $v_0$; in the rest of this work, for the Bradley-Terry model we will use $v_0=0.5$.

So, not very surprisingly, the similarity of the results for two different models is explained by the similarity of the models which happens due to our decision to use base-$10$ logarithm in the logistic function \eqref{F.Logistic}. 

\begin{figure}[bt]
\begin{tabular}{cc}
\scalebox{\sizf}{\includegraphics[width=\sizfs\linewidth]{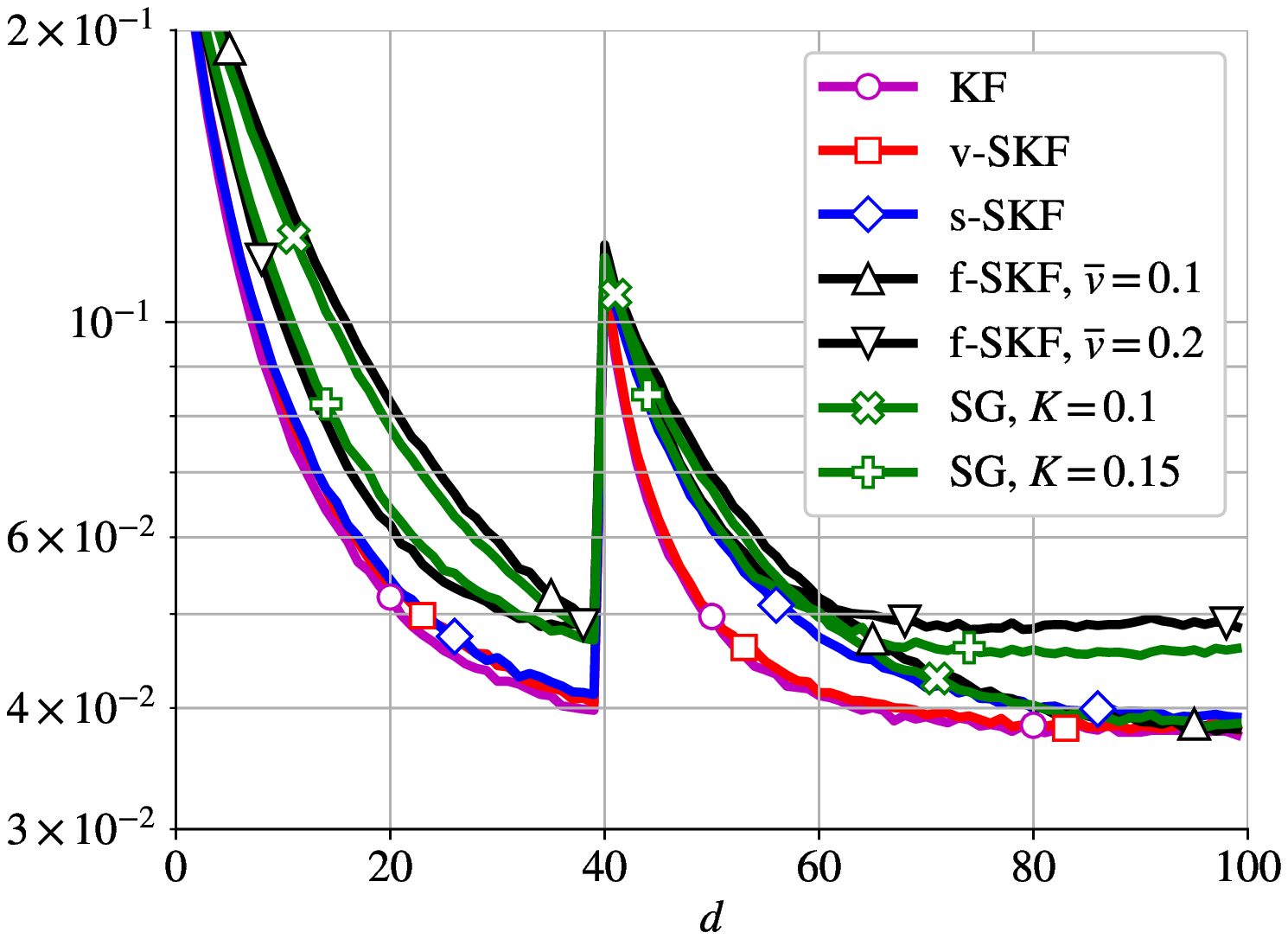}}
&
\scalebox{\sizf}{\includegraphics[width=\sizfs\linewidth]{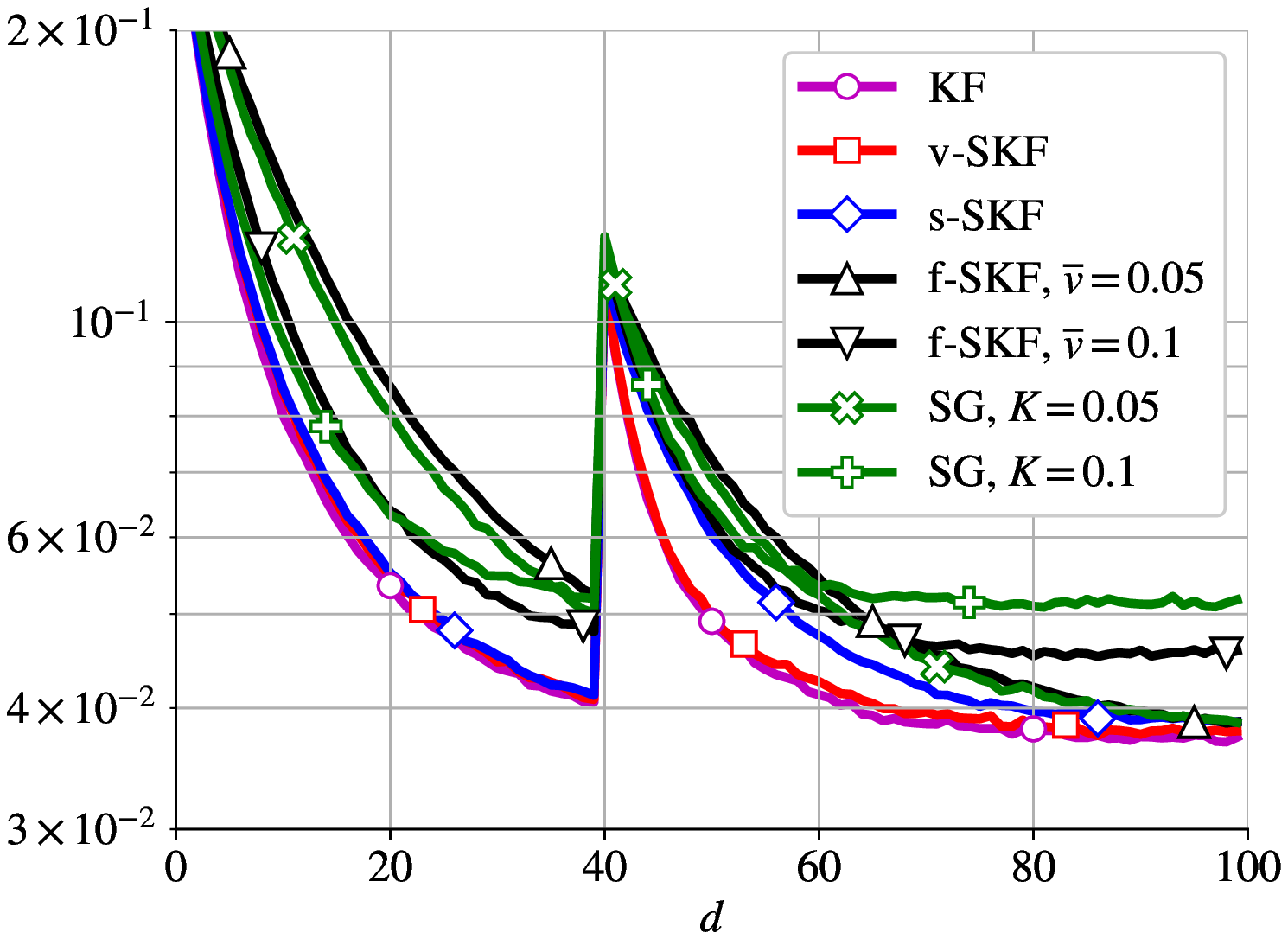}}\\
a) & b)
\end{tabular}
\caption{The average \gls{kl} divergence, when using the \gls{kf}, the \gls{vskf}, the \gls{sskf}, the \gls{fskf} and the \gls{sg} rating algorithms; $\beta=1$ and a) the Thurston model with $\epsilon=0.004$ and $v_0=1$; b) the Bradley-Terry model with $\epsilon=0.002$ and $v_0=0.5$.}\label{Fig:Thurston}
\end{figure}

We implement all the rating algorithms we proposed and comparing them in \figref{Fig:Thurston} we observe that:
\begin{itemize}
    \item Without surprise, the \gls{kf} ensures the best performance for both, the initialization phase (after $d=1$), and the post-switch phase (after $d=40$). 
    \item Rather surprisingly, the \gls{vskf} and the \gls{kf} ratings performs quasi-identically which suggests that the posterior correlation between the skills is not relevant even if the number of players in our example is moderate.
    \item The \gls{sskf} rating performs very well in the initialization phase because all the players have roughly the same variance and this is the assumption underpinning the algorithm. On the other hand, in the post-switch phase the variance of the players is disparate and then the convergence speed decreases.
    \item In the \gls{fskf} algorithm we may appreciate the trade-off: to increase the convergence speed we need larger $\ov{v}$, while the performance after convergence is improved with smaller $\ov{v}$. The value $\ov{v}$ which ensures the best performance after convergence may be deduced from the \gls{sskf} algorithm, where we have obtained $v_{T}\approx 0.1$ for the Thurston model and $v_T\approx 0.8$ for the Bradley-Terry model. Note again, that  this stays in line with the argument of matching the Gaussian and the logistic distributions: the relation between the posterior variances after convergence is close to $1/\sqrt{a}$.
    \item The \gls{sg} shares the drawbacks of the \gls{fskf} rating: larger $K$ improves the convergence speed at the expense of poorer performance after convergence. Note again the halving of the step size, $K$, for the Bradley-Terry model: remember, $K$ has the meaning of the variance and thus the same principle of matching the logistic and the Gaussian distributions we mentioned above applies. 
\end{itemize}

Overall, the important conclusions are:
\begin{itemize}
    \item 
        The \gls{vskf} algorithm is the best candidate for simple rating: it exploits the temporal model in the data and does not suffer loss comparing to the \gls{kf} rating,
    \item 
        Opting for further simplifications leads to some loss where, despite its simplicity, the \gls{sg} rating offers the performance comparable to other \gls{skf} ratings, and 
    \item 
        The model mismatch (applying the algorithms based on the Bradley-Terry model to the data generated using the Thurston model, see \figref{Fig:start}b) does not affect the performance in any significant manner. And while at first it may appear counter-intuitive, we should note that the performance of the algorithms is not evaluated by their ability to estimate the skills, $\btheta_t$, but rather by their predictive capability. Thus, using the Bradley-Terry model, the estimate $\bmu_t$ may be, indeed, far from the actual value of $\btheta_t$, because a different model if fitted to the data, but the prediction is little affected.
\end{itemize}

\begin{figure}[bt]
\begin{center}
\scalebox{\sizf}{\includegraphics[width=\sizfs\linewidth]{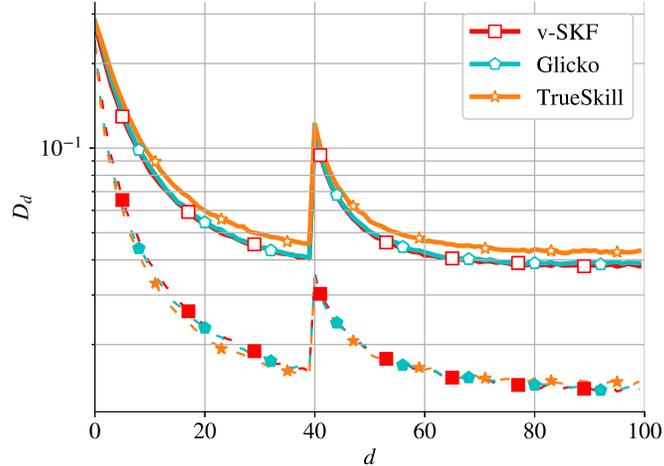}}
\end{center}
\caption{The average (solid line) and median (loosely dashed line) \gls{kl} divergence, when using the \gls{vskf} ($\epsilon=0.004$ and $v_0=1$), the TrueSkill ($\epsilon=0.004$ and $v_0=1$) and the Glicko ($\epsilon=0.002$ and $v_0=0.5$) rating algorithms; $\beta=1$.}\label{Fig:TrueSkill.Glicko}
\end{figure}

We apply now the TrueSkill and the Glicko algorithms to the same data set and show the results in \figref{Fig:TrueSkill.Glicko}, where we observe that the Glicko and the \gls{vskf} algorithms yield  practically indistinguishable results. This observations stays in line with the similarity of both algorithms which we observed in \secref{Sec:Glicko}. 

On the other hand, the TrueSkill algorithms, despite of being based on the same (Thurston) model which was used for data generation, suffers a small loss after convergence. This can be attributed to the adaptation step being increased comparing to the \gls{vskf} algorithm as we already noted in \secref{Sec:TrueSkill}. To put this difference in performance into perspective, we show also the median curve of the metrics; since the latter is much further from the mean than the differences among the algorithms, the ``loss'' of the TrueSkill may have no practical importance.

Before moving to the empirical data, we show in \figref{Fig:Thurston.noisy} the results of the \gls{vskf} and the \gls{sg} ratings obtained for different values of $\sigma=s$. Instead of the metric \eqref{DKL.metric}, which we will not be able to calculate in the empirical data, we show here the log-score 
\begin{align}\label{log.score.definition}
    \tr{LS}_t = - \sum_{y\in\mcY }\IND{y_t=y} \ell(\bx_{t}\T\bmu_{t-1}/s; y),
\end{align}
where $\mcY$ is the set of possible game outcomes. 

We see that increasing $\sigma$, \ie making the results more ``noisy'', the advantage of exploiting the temporal relationship between the skills is lost and the results obtained using the \gls{vskf} rating are very similar to those yield by the \gls{sg} rating. This leads to a cautionary note: if the uncertainty in the observations (that is, the ``noise") is large, the simple algorithms (such a the \gls{sg} rating) may provide satisfactory results.

\begin{figure}[bt]
\begin{center}
\scalebox{\sizf}{\includegraphics[width=\sizfs\linewidth]{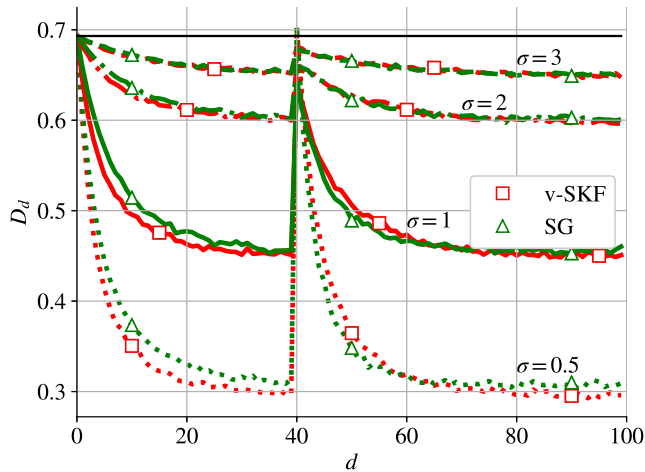}}
\end{center}
\caption{The average log-score obtained using for the Thurston model and the \gls{vskf} ($\epsilon=0.004$ and $v_0=1$) and the \gls{sg}  ($K=0.15\sigma$) ratings for different observation noise levels $\sigma$. The black horizontal line indicates the value $H=- \log 0.5\approx 0.69$ which is the entropy of uniformly distributed binary variable, see \eqref{H.definition}.}\label{Fig:Thurston.noisy}
\end{figure}

\subsection{Empirical data}\label{Sec:NHL}

We consider now the empirical results from
\begin{itemize}
    \item The ice-hockey games in the \gls{nhl} in the seasons 2005/06 -- 2014/15 except for the short season 2012/13. In this pre-expansion period, there were $M=30$ teams and the rules leading to the draws were kept the same.\footnote{Starting with the 2005/06 season, draws are not allowed and are resolved through shootouts if the game was tied in the overtime. Starting with the season 2015/16 the number of skaters in the overtime was changed from four to three.} We can thus treat the games as binary if we use the final result, or as ternary (\ie with draws) if we use the regulation-time results (before overtime/shootouts). Each team plays $82$ games so there are $T=1230=41M$ games in each season. 
    \item The football games of the \gls{epl} seasons  2009/10 -- 2018/19. There are $M=20$ teams, each playing $38$ games, thus $T=380$.
    
    \item The Americal footbal games of the \gls{nfl} in the seasons 2009/10 -- 2018/19. There are $M=32$ teams, each playing $16$ games, so $T=256$.
\end{itemize}

In the team games, the \gls{hfa} is present and, in the rating methods it is customary to take it into account by artificially ``boosting'' the skill of the home player (here, team): in all the functions taking $z_t/s$ as the argument we will rather use $z_t/s+\eta$, \eg in \eqref{pdf.y.theta} we use $L(z_t/s+\eta; y_t)$ instead of $L(z_t/s; y_t)$. The \gls{hfa} boost, $\eta$, must be found from data as we show in the following; note also that $\eta$ does not depend on the scale $s$. 

To consider the ``initialization" period we average
\eqref{log.score.definition} over the first $t_\tr{init}$ games
\begin{align}\label{LS.init}
    \ov{\tr{LS}}_\text{init} = \frac{1}{t_\text{init}}\sum_{t=1}^{t_\text{init}} \tr{LS}_t,
\end{align}
where, $t_\tr{init}=4 M$, which means that, in the first $t_\tr{init}$ games, each team played, on average, 8 times.

The performance after ``convergence" is evaluated by averaging  \eqref{log.score.definition} over the second half of the season
\begin{align}\label{LS.conv}
    \ov{\tr{LS}}_\tr{final} = \frac{2}{T}\sum_{t=T/2+1}^T \tr{LS}_t.
\end{align}
Further, we take the mean of \eqref{LS.init} and \eqref{LS.conv} over all seasons considered.

We will use the Bradley-Terry model in the binary games (in the \gls{nhl}) and  the \gls{hfa}-boost parameter $\eta$ is evaluated as \cite{Szczecinski20}
\begin{align}
    \eta=\log_{10}\frac{f_1}{f_0},
\end{align}
where $f_y$ is the estimated frequency of the game outcome $y\in\mcY$. Here, from the nine \gls{nhl} seasons under study we obtain $f_0\approx 0.45$ and $f_1\approx 0.55$, and thus $\eta= 0.08$.

For the ternary games we use the Davidson model \eqref{Davidson.L}-\eqref{Davidson.h}, and we estimate the home- and the draw parameters, $\eta$ and $\kappa$, using the strategy shown in \cite{Szczecinski20}, \cite{Szczecinski20c}
\begin{align}
    \eta &= \frac{1}{2}\log_{10} \frac{f_2}{f_0},\\
    \kappa &= \frac{f_1}{\sqrt{f_0 f_2}},
\end{align}
where, as before, $f_y$ are the frequencies of the events $y\in\set{0,1,2}$ estimated from the games in all seasons considered. These are i) for the \gls{nhl}: $f_0\approx 0.33$, $f_1\approx 0.24$, and $f_2\approx 0.43$,  ii) for the \gls{epl}: $f_0\approx 0.29$, $f_1\approx 0.25$, and $f_2\approx 0.46$, and iii) for the \gls{nfl}: $f_0\approx 0.43$, $f_1\approx 0.003$, and $f_2\approx 0.57$. The corresponding values of $\eta$ and $\kappa$ are shown in \tabref{tab:log-score}.

We consistently use $\beta=1$ and $s=1$; the parameters $v_0$, $\epsilon$ (for the \gls{vskf} algorithm), and the update step $K$ (for the \gls{sg} algorithm) which yield the best results are shown them in \tabref{tab:log-score}; they were found by scanning the space of admissible values. 

The log-score results shown in  \tabref{tab:log-score} may be compared to the log-score of the prediction based on the frequencies of the events $y_t$, \ie
\begin{align}\label{H.definition}
    H = - \sum_{y\in\mcY} f_y \log f_y,
\end{align}
which is the entropy calculated from the estimated frequencies.

\begin{table}[tb]
    \centering
    \begin{tabular}{c | c | c| c |c| c}
       \multicolumn{2}{c|}{} 
       & NHL   &   NHL  & EPL & NFL\\
       \multicolumn{2}{c|}{}  
        & Bradley-Terry  & Davidson & Davidson & Davidson\\
        \multicolumn{2}{c|}{}  
        & $\eta=0.08$  & $\eta=0.05$, $\kappa=0.63$ & $\eta=0.10$, $\kappa=0.67$   & $\eta=0.06$, $\kappa=5.5\cd 10^{-3}$\\
    \hline
    \multirow{3}{*}{v-SKF}   &  ($v_0$, $\epsilon$) &
    ($0.01$, $3\cd 10^{-5}$) &  ($0.003$, $3\cd 10^{-5}$)  &    ($0.04$, $10^{-7}$) & ($0.02$, $10^{-4}$)\\
    &  $\ov{\tr{LS}}_\tr{init}$ &
    $0.688$ &  $1.063$  &    $1.055$ & $0.679$\\
    &  $\ov{\tr{LS}}_\tr{final}$ &
    $0.678$ &  $1.064$  & $0.974$ & $0.640$\\
    \hline
    \multirow{3}{*}{SG}      &  ($K$) &
    ($0.01$) &  ($0.003$) & ($0.015$) & ($0.015$)\\
    &  $\ov{\tr{LS}}_\tr{init}$ &
    $0.688$ &  $1.063$ & $1.052$ & $0.678$\\
    &  $\ov{\tr{LS}}_\tr{final}$ &
    $0.678$ & $1.064$  & $0.976$ &  $0.641$\\
    \hline
    \multicolumn{2}{c|}{$H$}         &  $0.688$ &  $1.071$  &    $1.061$  & $0.700$  
    \end{tabular}
    \caption{Log-score obtained in the \gls{nhl}, the \gls{nfl}, and the \gls{epl} games using the the \gls{vskf} and the \gls{sg} algorithms ( the \gls{kf} and the \gls{kf} ratings yield the same results). The entropy, $H$, calculated from  \eqref{H.definition} is shown as a reference. The Bradley-Terry model corresponds to the binary games (in \gls{nhl}), while the Davidson model takes into account the ternary outcomes. Due to a very small frequency of draws in the \gls{nfl}, the results are practically binary but the presence of the draws affects the entropy which exceeds the limit for the binary variable.
    }
    \label{tab:log-score}
\end{table}

The conclusions drawn from the synthetic data also hold here: the performance of the \gls{vskf} and the \gls{kf} algorithms is virtually the same. The \gls{sg} rating is taken as the representative of other simplified algorithms.

We observe in \tabref{tab:log-score} that the predictions in the \gls{nhl} are merely better than the entropy and, referring to \figref{Fig:Thurston.noisy} we might attribute it to the ``noisy'' game outcomes, which would also explain why the results produced by the \gls{vskf} and the \gls{sg} algorithms are virtually the same, and why we cannot see any differences even in the  initialization phase of the algorithm. 

By the same token, we can say that the noise decreases in \gls{nfl} results, and more so in the \gls{epl} ones: so we can distinguish between the performance of the initialization and after the convergences. Yet, the improvement due to the use of the \gls{vskf} and the \gls{kf} algorithms is still negligible. This also can be intuitively understood from the parameters we found to minimize the average log-score. Note that, for the \gls{epl} we use the variance $v_0=0.04$ and $s=1$, but, applying \propref{Prop:SKF.scale} we  might equally well use $v_0=1$ and $s=5$; the latter scenario may be related to a model with large outcome noise $\sigma$ as shown in \figref{Fig:Thurston.noisy}.

The difference between the \gls{vskf} and the \gls{sg} algorithms may be also appreciated by inspecting the temporal evolution of $\bmu_t$, shown in \figref{Fig:Trajectories} and obtained for the 2009/10 \gls{epl} season. While the differences in the log-score results shown in \tabref{tab:log-score} are rather small, we can appreciate that the skills estimated using the \gls{vskf} converge very fast to the final values (after 50 days, approx.), to which the \gls{sg} rating also converges but the time required is longer (200 days, approx.); this effect is particularly notable for the teams with extreme values of the means, that is, for very strong, as well as very weak teams.

\begin{figure}
    \centering
    \scalebox{\sizf}{\includegraphics[width=\sizfs\linewidth]{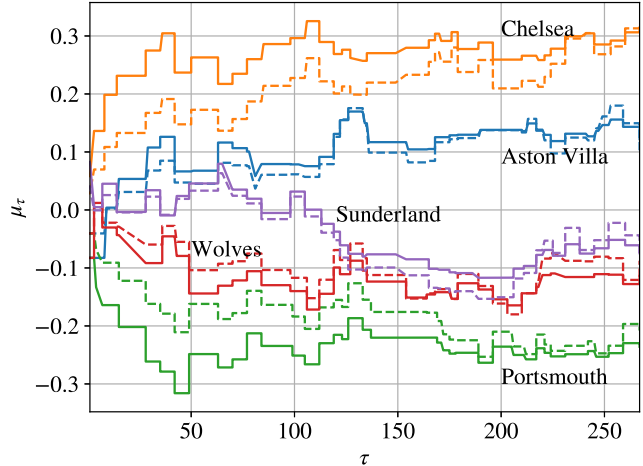}}
    \caption{Evolution of the (selected) means $\mu_\tau$ indexed with the time-stamp $\tau(t)$ in the 2009-10 \gls{epl} season obtained using the \gls{vskf} (solid lines) and the \gls{sg} (dashed lines) rating algorithms.}
    \label{Fig:Trajectories}
\end{figure}

\section{Conclusions}\label{Sec:Conclusions}

In this work we propose a class of online Bayesian rating algorithms for one-on-one games which can be used with any skills-outcome model and which encompasses the case of the group sports, a case typically encountered in eSports. By using various simplifications to represent the posterior covariance matrix in the Gaussian distributions, we obtain different algorithm in the same class. 

Deriving such a generic algorithms should not only streamline the passage from the skills-outcome model to the actual online algorithm but also provides a fresh insight into the relationship between the existing rating methods such as the Elo, the Glicko and the TrueSkill algorithms. Their differences and similarities are discussed and we demonstrate that, the Glicko and the TrueSkill algorithms may be seen as instances of our generic algorithms.  This is an interesting observation in its own right as it unifies the view on these two popular algorithms, which, even if derived  from different principles, are now shown in a common framework. We also provide a new insight into the interpretation of the Elo algorithm.

We show numerical examples applied to the synthetic-- and the empirical data, which provide guidelines about the conditions under which the algorithms should be used. In particular, our results indicate that the differences between the \gls{kf} rating (with a full representation of the covariance matrix) and the \gls{vskf} rating (only diagonal of the covariance is preserved) are negligible. The \gls{vskf} is, in fact, very similar to the Glicko and the TrueSkill algorithms. 

We  show that further simplification of the covariance matrix may be counterproductive and the simple, \acrfull{sg} rating may be then a competitive solution, even though it cannot be treated as a Bayesian algorithm as it only provides the point estimate of the skills. The simple \gls{sg}-based rating is indeed appealing and particularly useful in very noisy data, \ie when the prediction of the game outcomes cannot be reliably inferred from the estimated skills. 

These observations, made in the synthetic setup are then confirmed in empirical data, where we analyse the game results from the professional hockey, the American football, and the association football games. Indeed, the differences between the \gls{vskf} and the \gls{sg} results are, at best, small but notable (in football, where the data is relatively not noisy) and, at worst, negligible (in hockey, where the game outcomes are very noisy). In fact, the very concept of observational noise in the sport outcomes received very little attention in the literature and we believe studying it in more depth is an interesting research venue.

Overall conclusion regarding the applicability of the algorithms is that, in reliable (not noisy) data,  the online Bayesian rating algorithm may provide improved convergence, and potential skill-tracking capability. On the other hand, the reality of sport competition outcomes may not conform to these requirements and, when dealing with the noisy observations, the simple algorithms, such as the Elo rating (which is a an instantiation of the \gls{sg} rating) may be equally useful.

\begin{appendices}

\section{Proof of \propref{Prop:DKL}}\label{Proof:DKL}

Our goal is to find the Gaussion distribution $\tilde{f}(\btheta)=\mcN(\btheta;\bmu,\bV)$ under the form  \eqref{covariance.cases} minimizing the \gls{kl} divergence with a given distribution $f(\btheta)$
\begin{align}
    D_\tr{KL}\big(f|| \tilde{f}\big)
    &=\int f(\btheta)\log\frac{f(\btheta)}{\tilde{f}(\btheta)} \dd \btheta\\
    \label{DKL.mu.V}
    &\propto\frac{1}{2}\log\tr{det}(2\pi\bV)
    +\frac{1}{2}\int f(\btheta)(\btheta-\bmu)\T\bV^{-1}(\btheta-\bmu)\dd \btheta.
\end{align}
Gradient of \eqref{DKL.mu.V} with respect to $\bmu$ is zeroed for $\bmu=\Ex[\btheta]$, and this, irrespectively of the form of $\bV$, this proves \eqref{bmu.yt}. This is a well-known result, as well, as the one which says that, to minimize \eqref{DKL.mu.V} we also have to use $\bV=\tr{Cov}[\btheta]=\Ex[(\btheta-\bmu)\T(\btheta-\bmu)]$, this is the claim in  \eqref{bV.yt}.

Now assume that we use the vector-covariance model, \ie we have to find $\tilde{f}(\btheta)=\mcN(\btheta;\bmu,\tr{diag}(\bv))$. Then, \eqref{DKL.mu.V} becomes
\begin{align}\label{DKL.vector}
 D_\tr{KL}\big(f|| \tilde{f}\big)
 \propto \frac{1}{2}\sum_{m=1}^M \log v_m + \sum_{m=1}^M \frac{\tr{Var}[\theta_m]}{2v_m},
\end{align}
where $\tr{Var}[\theta_m]$ is the variance of $\theta_m$.

Zeroing the derivative of \eqref{DKL.vector} with respect to $v_m$ yields $v_m=\tr{Var}[\theta_m]$, that is,  $\bv=\tr{di}(\tr{Cov}[\btheta])$ which proves \eqref{bv.t.DKL}.

Finally, if we adopt scalar-covariance model $\tilde{f}(\btheta)=\mcN(\btheta;\bmu,v\bI)$, \eqref{DKL.vector} becomes
\begin{align}\label{DKL.scalar}
D_\tr{KL}\big(f|| \tilde{f}\big)
 \propto \frac{M}{2} \log v + \frac{1}{2v} \sum_{m=1}^M \tr{Var}[\theta_m],
\end{align}
whose derivative with respect to $v$ is zeroed if $v=\frac{1}{M}\sum_{m=1}^M\tr{Var}[\theta_m]$, and this proves \eqref{v.t.DKL}.

\section{Proof of \propref{Prop:SKF.scale}}\label{Proof:SKF.scale}

For brevity, let us use the symbol $\check{(\cd)}$ to denote only the scaled variables, \eg $\check{\bV}\equiv\bV(s,s^2v_0,s^2\epsilon)$; the unscaled  ones are used without the symbols, \eg $\bV\equiv \bV(1,v_0,\epsilon)$.

The proof is done by induction: by construction, the initialization satisfied the Proposition, \ie $\check{\bV}_0=s^2\bI=s^2\bV_0$,  and we suppose that $\check{\bV}_{t-1}=s^2 \bV_{t-1}$ and $\check{\bmu}_{t-1}=s \bmu_{t-1}$ hold. Then we must have  $\beta_{t}\bx_t\T\check{\bmu}_{t-1}=s\beta_{t}\bx_t\T\bmu_{t-1}$, and $g_t$ and $h_t$ are not affected by scaling. Then we also have 
$\check{\ov{\bV}}_t=\beta^2_t\check{\bV}_{t-1}+s^2\epsilon_t\bI =s^2 \ov{\bV}_t$ and $\check{\omega}_t=s^2\omega_t$, so \eqref{oneshot.mean.KF} may be written as
\begin{align}
    \check{\bmu}_t
    &=\beta_{t} \check\bmu_{t-1} + \check{\ov{\bV}}_t\bx_t\frac{sg_t}{s^2+h_t\check\omega_t}\\
    &= \beta_{t} s \bmu_{t-1} + s^2\ov{\bV}_t\bx_t\frac{sg_t}{s^2+s^2h_t\omega_t}
    =s\bmu_{t},
\end{align}
and \eqref{Vt.update.KF}, as 
\begin{align}
    \check{\bV}_t
    &=\check{\ov{\bV}}_{t-1} + \check{\ov{\bV}}_{t-1}\bx_t\bx_t\T\check{\ov{\bV}}_{t-1}\frac{h_t}{s^2+h_t\check\omega_t}\\
    &= s^2\ov{\bV}_{t-1} + s^4\ov{\bV}_{t-1}\bx_t\bx_t\T\ov{\bV}_{t-1}\frac{h_t}{s^2+s^2h_t\omega_t}=s^2 \bV_t.
\end{align}
This ends the proof for the \gls{kf} algorithm. By extension all other algorithms derived from the \gls{kf} algorithm must satisfy the claims of \propref{Prop:SKF.scale}, which may be also proven with the steps shown above applied to the \gls{vskf}, \gls{sskf}, and \gls{fskf} algorithms.
\end{appendices}


\end{document}